\documentclass{article}
\usepackage{amsfonts,amsmath,amsthm,amssymb}

\date{}

\usepackage{algorithm}
\usepackage{algorithmic}
\usepackage{amsthm}
\usepackage{subfigure}
\usepackage{epsfig}
\usepackage{graphicx}
\usepackage{url}
\usepackage{multirow}
\usepackage{amssymb}
\usepackage{mathrsfs}
\usepackage{epstopdf}
\usepackage{multicol}
\usepackage{amsmath}
\usepackage{mathtools}
\usepackage{bm}
\usepackage{color}
\usepackage{multicol}
\usepackage{wrapfig}
\usepackage{wrapfig,lipsum,booktabs}
\usepackage{enumitem}
\usepackage{multirow}
\usepackage{makecell}

\usepackage{pifont}
\newcommand{\cmark}{\ding{51}}%
\newcommand{\xmark}{\ding{55}}%

\allowdisplaybreaks[4]

\newtheorem{theorem}{Theorem}[section]

\newtheorem{lemma}[theorem]{Lemma}
\newtheorem{corollary}[theorem]{Corollary}

\newtheorem{assumption}[theorem]{Assumption}
\newtheorem{remark}[theorem]{Remark}
\numberwithin{equation}{section}

\theoremstyle{plain}

\theoremstyle{definition}

\title{On Provable Benefits of Muon in Federated Learning}

\author{
	Xinwen Zhang \thanks{Temple University, {\tt ellenz@temple.edu}} \and 
	Hongchang Gao\thanks{Temple University, {\tt hongchang.gao@temple.edu}} 
	}
\begin{document}
\maketitle

\begin{abstract}
	The recently introduced optimizer, Muon, has gained increasing attention due to its superior performance across a wide range of applications. However, its effectiveness in  federated learning  remains unexplored. To address this gap, this paper investigates the performance of Muon in the federated learning setting. Specifically, we propose a new algorithm, FedMuon, and establish its convergence rate for nonconvex problems. Our theoretical analysis reveals multiple favorable properties of FedMuon. In particular, due to its orthonormalized update direction, the learning rate of FedMuon is independent of problem-specific parameters, and, importantly, it can naturally accommodate heavy-tailed noise. The extensive experiments on a variety of neural network architectures validate the effectiveness of the proposed algorithm.
\end{abstract}

\section{Introduction}
Recently, several new optimizers have been developed based on various inductive biases regarding machine learning models. Among them, Muon \cite{jordan2024muon} has gained more attention due to its superior performance across a wide range of applications. Muon is essentially a stochastic gradient descent with momentum (SGDM) algorithm. Its key difference from traditional SGDM lies in directly optimizing a two-dimensional matrix, rather than flattening it into a vector, using an orthonormalized momentum matrix. Specifically, assuming the momentum of the stochastic gradient in the $t$-th iteration is denoted by $M_{t}\in\mathbb{R}^{m\times n}$,  Muon orthonormalizes it as follows:
\begin{equation}
	  O_{t} = \arg\min_{O} \|O- M_t\|^2_F, \quad s.t. \quad O^TO=I_{n} \ , 
\end{equation}
where $I_{n}\in\mathbb{R}^{n\times n}$ denotes the identity matrix.  The optimal solution to this problem is $O_{t}=U_tV_t^T$ where $U_t\in \mathbb{R}^{m\times r}$ and $V_t\in \mathbb{R}^{n\times r}$ are obtained from the  singular value decomposition (SVD) of $M_t$, i.e., $M_t=U_tS_tV_t^T$. Here,  $S_t\in \mathbb{R}^{r\times r}$ is a diagonal matrix whose diagonal entries are the singular values of  $M_t$, and $r$ denotes the rank of $M_t$.  Muon  proposes using  Newton–Schulz approach \cite{bernstein2024old} to approximately solve this problem, instead of SVD, in order to accelerate computation.  Due to its orthonormalization step, Muon has demonstrated strong performance across a wide range of applications, such as the pretraining of large language models \cite{liu2025muon}.

The theoretical convergence rate of Muon has been well studied this year \cite{li2025note,shen2025convergence,an2025asgo,kovalev2025understanding,zhang2025adagrad,sato2025analysis,sfyraki2025lions,chen2025muon}. For example, \cite{li2025note} provided the first convergence analysis for Muon when the loss function is nonconvex. \cite{shen2025convergence} established the convergence rate of Muon when the loss function is nonconvex and star-convex.  However, all these existing works focus solely on the single-machine setting, making it unclear how well Muon performs in the federated learning context. Federated learning \cite{mcmahan2017communication} is an important distributed machine learning framework that enables model training on decentralized data without sharing raw data. On the other hand, the orthonormalization step in Muon introduces new properties to the search direction, such as bounded magnitude. This naturally leads to the question: \textbf{how does Muon perform in the federated learning setting? Specifically, what convergence rate and communication complexity can Muon achieve in this context?}

To answer this question, we first develop a new federated optimization algorithm, FedMuon, which employs Muon to update variables on each worker and periodically communicates these updates to the central server. Then, we established the convergence rate of FedMuon for nonconvex problems under mild assumptions. Specifically, our theoretical analyses show that FedMuon enjoys the following favorable properties:
\begin{itemize}[left=10pt]
	\vspace{-0.1in}
	\item The learning rate of FedMuon is inherently independent of problem-specific parameters, such as the Lipschitz constant (see Remark~\ref{remark:parameter-free}).
	\item FedMuon can naturally accommodate heavy-tailed noise, as it does not require gradient clipping to guarantee convergence (see Section~\ref{section:heavy-tailed-noise}). 
\end{itemize}
\vspace{-0.05in}
The detailed comparison between FedMuon and existing state-of-the-art methods can be found in Table~\ref{table:comparison}. All these favorable properties are due to the orthonormalization operation in Muon. To the best of our knowledge, this is the first work revealing these favorable properties of Muon in federated learning. Finally, we performed extensive experiments to validate the performance of our new algorithm and the experimental results confirm the efficacy of FedMuon.

\section{Related Work}
\subsection{Federated Optimization}
To solve federated learning models, numerous federated optimization algorithms \cite{mcmahan2017communication,stich2018local,yu2019parallel,yu2019linear,yang2021achieving,khanduri2021stem,wu2023faster} have been proposed and analyzed in the past few years. For example, \cite{yu2019parallel} established a convergence rate of $O(1/\epsilon^4)$ and a communication complexity of $O(1/\epsilon^3)$ for LocalSGD in nonconvex optimization problems by relying on a bounded gradient norm, where $\epsilon>0$ denotes the solution accuracy. \cite{yu2019linear} established the same convergence rate and communication complexity for LocalSGD with momentum (LocalSGDM) in nonconvex optimization problems. Unlike LocalSGD, it does not require a bounded gradient norm but instead relies on a bounded heterogeneity assumption. Under the same heterogeneity assumption as in \cite{yu2019linear}, \cite{khanduri2021stem} proposed STEM, which uses the stochastic variance-reduced gradient to improve the convergence rate to $O(1/\epsilon^3)$ and communication complexity to $O(1/\epsilon^2)$ for nonconvex problems.

\begin{table*}[t] 
	\small
	\setlength{\tabcolsep}{1pt}
	\begin{center}
		\caption{ The comparison of convergence rate and communication complexity of different federated optimization algorithms for nonconvex problems. Note that all these algorithms can achieve linear speedup, so we omit this for simplicity. In the first column, \textbf{M} denotes momentum, \textbf{V} denotes variance reduction. }
		\begin{tabular}{l|c|c|c|c|c}
			\toprule
			\multirow{2} * {} &   {\textbf{Algorithms}} &  \textbf{Convergence}   & \textbf{Communication} &  \textbf{Parameter}  & \textbf{Heavy-tailed} \\
			& {} & \textbf{Rate} & \textbf{Complexity}  & \textbf{Free} & \textbf{Noise} \\
			\hline
			\multirow{4} * {-} & \makecell{\textbf{FedAvg/LocalSGD} \\ \cite{yu2019parallel}} & $O\left(\frac{1}{\epsilon^4}\right)$  & $O\left(\frac{1}{\epsilon^3}\right)$ & \xmark & \xmark \\
			& \makecell{\textbf{SCAFFOLD} \\ \cite{karimireddy2020scaffold}} & $O\left(\frac{1}{\epsilon^4}\right)$  & $O\left(\frac{1}{\epsilon^3}\right)$  & \xmark & \xmark \\
			\hline
			\multirow{8} * {M} & \makecell{\textbf{LocalSGDM} \\ \cite{yu2019linear} }    & $O\left(\frac{1}{\epsilon^4}\right)$  & $O\left(\frac{1}{\epsilon^3}\right)$  & \xmark & \xmark \\
			& \makecell{ \textbf{FedAvg-M} \\ \cite{cheng2024momentum}} &  $O\left(\frac{1}{\epsilon^4}\right)$  & $O\left(\frac{1}{\epsilon^3}\right)$ & \xmark & \xmark\\
			& \makecell{ \textbf{SCAFFOLD-M} \\ \cite{cheng2024momentum}} & $O\left(\frac{1}{\epsilon^4}\right)$  & $O\left(\frac{1}{\epsilon^3}\right)$  & \xmark & \xmark\\
			& \makecell{ \textbf{PAdaMFed} \\ \cite{yan2025problemparameterfree}} & $O\left(\frac{1}{\epsilon^4}\right)$  & $O\left(\frac{1}{\epsilon^3}\right)$  & \cmark & \xmark\\
			\hline
			\multirow{8} * {V} & \makecell{\textbf{STEM} \\\cite{khanduri2021stem}} & $O\left(\frac{1}{\epsilon^3}\right)$  & $O\left(\frac{1}{\epsilon^2}\right)$ & \xmark & \xmark\\
			& \makecell{\textbf{FedAvg-VR} \\ \cite{cheng2024momentum}}  &  $O\left(\frac{1}{\epsilon^3}\right)$  & $O\left(\frac{1}{\epsilon^2}\right)$  & \xmark & \xmark\\
			& \makecell{\textbf{SCAFFOLD-VR} \\ \cite{cheng2024momentum}}  &  $O\left(\frac{1}{\epsilon^3}\right)$  & $O\left(\frac{1}{\epsilon^2}\right)$ & \xmark & \xmark\\
			& \makecell{\textbf{PAdaMFed-VR} \\ \cite{yan2025problemparameterfree}}  &  $O\left(\frac{1}{\epsilon^3}\right)$  & $O\left(\frac{1}{\epsilon^2}\right)$  & \cmark & \xmark \\
			\hline
			\multirow{4} * {M} & \makecell{\textbf{FedMuon}\\ Corollary~\ref{corollary:regular-noise-epsilon}}&  $O\left(\frac{1}{\epsilon^4}\right)$  & $O\left(\frac{1}{\epsilon^3}\right)$ & \cmark & \xmark \\
			& \makecell{\textbf{FedMuon}\\ Corollary~\ref{corollary:heavy-tailed-epsilon}}&  $O\left(\frac{1}{\epsilon^{\frac{2p}{p-1}}}\right)$  & $O\left(\frac{1}{\epsilon^{\frac{3p}{2(p-1)}}}\right)$ & \cmark & \cmark \\
			\bottomrule 
		\end{tabular}
	\end{center}
	\label{table:comparison}
	\vspace{-10pt}
\end{table*}

To mitigate the influence of heterogeneous data distributions, a couple of federated optimization algorithms \cite{karimireddy2020scaffold,cheng2024momentum,yan2025problemparameterfree} have been developed to establish convergence rates without making any assumptions about heterogeneity. Essentially, these methods introduce a global control variate to mitigate heterogeneity. For example,  \cite{karimireddy2020scaffold} proposed SCAFFOLD, which uses a global control variate to adjust the local stochastic gradient, and established established its convergence rate for both strongly convex and nonconvex problems.  In particular, this algorithm achieves the same convergence rate and communication complexity as LocalSGD and  LocalSGDM for nonconvex problems. Later, \cite{cheng2024momentum} leveraged  variance reduced techniques to improve the convergence rate to $O(1/\epsilon^3)$ and the communication complexity to $O(1/\epsilon^2)$. Building on this strategy,  \cite{yan2025problemparameterfree} developed a problem-parameter-free algorithm, whose learning rate does not rely on problem-specific parameters, and established the same convergence rate and communication complexity as \cite{cheng2024momentum}. However, all these federated optimization algorithms that do not rely on the heterogeneity assumption require communication of a global control variate, which introduces additional communication overhead in each round. 
On the other hand, to handle the heavy-tailed noise, \cite{lee2025efficient} proposed using the gradient clipping technique on each worker to mitigate the influence of heavy-tailed noise. However, the gradient clipping approach requires a threshold to clip gradients, which is difficult to tune in practical applications. 

\subsection{Muon}
Muon was first proposed in \cite{jordan2024muon} to optimize the hidden layer of deep neural networks, which showed great performance for various applications. Several recent works \cite{li2025note,an2025asgo,kovalev2025understanding,shen2025convergence,riabinin2025gluon,zhang2025adagrad,sato2025analysis,sfyraki2025lions,chen2025muon,pethick2025training} have attempted to establish its convergence rate in the single-machine setting. In particular, \cite{li2025note} established the convergence rate of Muon for nonconvex problems under the assumption of Frobenius-norm Lipschitz smoothness. \cite{an2025asgo} provided its convergence rate under a generalized-norm Lipschitz smoothness assumption. \cite{kovalev2025understanding} further analyzed Muon's convergence rate given the spectral-norm Lipschitz smoothness assumption. The recent work \cite{shen2025convergence} provided convergence analysis for Muon under all these smoothness assumptions when the loss function is nonconvex and star-convex.  In addition, \cite{chen2025muon} established the convergence rate of Muon from the perspective of spectral norm constraints.  \cite{zhang2025adagrad} combined Muon with Adagrad to introduce the adaptive learning rate and then established its convergence rate for nonconvex problems. Moreover,  \cite{sfyraki2025lions} established the convergence rate of Muon from the perspective of the Frank-Wolfe method for nonconvex problems. It then introduced the gradient clipping technique to Muon to handle the heavy-tailed noise.  In this paper, we will show that FedMuon can still guarantee convergence without relying on the clipping operation.

\section{Problem Setup}\label{sec:setup}
\subsection{Problem Definition}
In this paper,  $K$ (where $K>0$) workers collaboratively optimize the following  problem:  
\begin{equation}
	\min_{X \in \mathbb{R}^{m\times n}} \frac{1}{K}\sum_{k=1}^{K} f^{(k)}(X) \ , 
\end{equation}
where $f^{(k)}(X)=\mathbb{E}[f^{(k)}(X; \xi)]$,  $X \in \mathbb{R}^{m\times n}$ denotes the optimization variable, and the superscript $k\in \{1, \cdots, K\}$ denotes the index of workers.  In the federated learning setting, all workers communicate with a central server to exchange updated variables or gradients.

In this paper, for a matrix $X\in \mathbb{R}^{m\times n}$,  $\|X\|_F$ denotes the Frobenius norm, $\|X\|_*$ denotes the nuclear norm, and  $\|X\|_2$ denotes the spectral norm. In addition, for  $X\in \mathbb{R}^{m\times n}$ and  $Y\in \mathbb{R}^{m\times n}$, we have $\langle X, Y\rangle = \text{Tr}(X^TY)$, where $\text{Tr}(\cdot)$ denotes the trace of a matrix.  Moreover, $\bar{X} =\frac{1}{K}\sum_{k=1}^{K}X^{(k)}$. 

\subsection{Assumption}
In this paper, we introduce the following assumptions, which has been commonly used in \cite{li2025note,shen2025convergence,zhang2025adagrad,sato2025analysis,sfyraki2025lions}.
\begin{assumption} \label{assumption:smoothness}
	For any $k\in\{1, \cdots, K\}$, the loss function $f^{(k)}(\cdot)$ is $L$-smooth, i.e., for any $X_1\in \mathbb{R}^{m\times n}$ and $X_2\in \mathbb{R}^{m\times n}$, it satisfies: $\|\nabla f^{(k)}(X_1) -\nabla f^{(k)}(X_2) \|_F \leq L \|X_1 - X_2\|_F$, where $L>0$ is a constant. 
\end{assumption}

\begin{assumption}\label{assumption:regular-noise}
	For any $k\in\{1, \cdots, K\}$, the stochastic gradient $\nabla f^{(k)}(X; \xi)$ is an unbiased estimator of the full gradient and satisfies: $\mathbb{E}[\|\nabla f^{(k)}(X; \xi) - \nabla f^{(k)}(X) \|_F^2] \leq \sigma^2$, 	where  $\sigma>0$ is a constant. 
\end{assumption}

\begin{assumption} \label{assumption:heavy-tailed-noise}
	For any $k\in\{1, \cdots, K\}$, the stochastic gradient $\nabla f^{(k)}(X; \xi)$ is an unbiased estimator of the full gradient and satisfies: $\mathbb{E}[\|\nabla f^{(k)}(X; \xi) - \nabla f^{(k)}(X) \|^p_F] \leq \sigma^p$, where  $\sigma>0$ is a constant and $p\in(1, 2]$. 
\end{assumption}
Note that Assumption~\ref{assumption:heavy-tailed-noise} characterizes heavy-tailed noise. When $p=2$, it reduces to the standard bounded noise assumption in Assumption~\ref{assumption:regular-noise}. 

\begin{assumption}\label{assumption:heterogeneity}
	The gradient satisfies the heterogeneity condition: $\frac{1}{K}\sum_{k=1}^{K} \mathbb{E}[\|\nabla f^{(k)}(X) - \nabla f(X) \|_F^2] \leq \delta^2$, where $\delta>0$ is a constant. 
\end{assumption}
Note that Assumption~\ref{assumption:heterogeneity} implies a heterogeneous data distribution when $\delta>0$, and a homogeneous data distribution when $\delta=0$. This assumption is widely used in existing federated learning literature, such as \cite{yu2019linear}. 

\begin{algorithm}[ht]
	\caption{FedMuon}
	\label{alg:fedmuon}
	\begin{algorithmic}[1]
		\REQUIRE $\eta>0$,  $\beta>0$, $\tau>1$. \\
		\STATE $M^{(k)}_{0} =  \nabla f^{(k)}(X_{0}^{(k)}; \xi_{0}^{(k)})$ \\
		\FOR{$t=0,\cdots, T-1$, the $k$-th worker} 
		
		\STATE  $(U^{(k)}_{t}, S^{(k)}_{t}, V^{(k)}_{t})=\text{SVD}(M^{(k)}_{t})$ \quad // Orthonormalize $M^{(k)}_{t}$ with  Newton–Schulz approach  
		\STATE	$X^{(k)}_{t+1} =  X^{(k)}_{t} -\eta U^{(k)}_{t}(V^{(k)}_{t})^T$\quad // Update variable  $X^{(k)}_{t}$ \\
		
		\STATE $M^{(k)}_{t+1} = (1-\beta) M^{(k)}_{t} + \beta \nabla f^{(k)}(X_{t+1}^{(k)}; \xi_{t+1}^{(k)})$ \quad  // Update gradient momentum $M^{(k)}_{t}$ \\

		\IF {$\text{mod}(t+1, \tau)==0$}
		\STATE $X^{(k)}_{t+1}=\frac{1}{K}\sum_{k'=1}^{K}X^{(k')}_{ t+1}$, \  $M^{(k)}_{t+1}=\frac{1}{K}\sum_{k'=1}^{K}M^{(k')}_{ t+1}$ \quad // Perform communication
		\ENDIF

		\ENDFOR
	\end{algorithmic}
\end{algorithm}

\section{Algorithm}
In Algorithm~\ref{alg:fedmuon}, we developed a novel federated optimization algorithm based on Muon, i.e., FedMuon.  In the $t$-th iteration, as shown in Step 5 of Algorithm~\ref{alg:fedmuon} each worker $k$ uses its local training samples to update the momentum $M^{(k)}_{t}\in \mathbb{R}^{m\times n}$ as follows:
\begin{equation}\label{eq:momentum}
	\small M^{(k)}_{t+1} = (1-\beta) M^{(k)}_{t} + \beta \nabla f^{(k)}(X_{t+1}^{(k)}; \xi_{t+}^{(k)}) \  , 
\end{equation} 
In Eq.~(\ref{eq:momentum}),  $\beta\in (0, 1)$ denotes the hyperparameter.  $\nabla f^{(k)}(X_{t}^{(k)}; \xi_{t}^{(k)})$ denotes the stochastic gradient, where $X_{t}^{(k)}$ denotes the variable  and $\xi_{t}^{(k)}$ represents the randomly selected training samples on the $k$-th worker in the $t$-th iteration.  In Step 3, each worker uses Newton–Schulz approach to perform SVD for $M^{(k)}_{t}$, where 
$U^{(k)}_{t}\in \mathbb{R}^{m\times r}$ is composed of left singular vectors, $V^{(k)}_{t}\in \mathbb{R}^{n\times r}$ consists of the right singular vectors, and $S^{(k)}_{t}$ is a diagonal matrix of singular values. With such a decomposition, FedMuon updates variable as follows:
\begin{equation}
	X^{(k)}_{t+1} =  X^{(k)}_{t} -\eta U^{(k)}_{t}(V^{(k)}_{t})^T \ , 
\end{equation}
where $\eta>0$ denotes the learning rate.  Every $\tau>1$ iterations, as shown in Step 7 of Algorithm~\ref{alg:fedmuon}, each worker $k$ uploads its local variable $X^{(k)}_{t+1}$ and momentum $M^{(k)}_{t+1}$ to the central server. The  server then averages all received variables and momentum and broadcasts the result to all workers.

\section{Theoretical Results}
In this section, we establish the convergence rate of FedMuon under different assumptions regarding gradient noise. 
\subsection{Convergence Rate  Under Bounded Variance}
\begin{theorem} \label{theorem:regular-noise}
	Given Assumptions~\ref{assumption:smoothness},~\ref{assumption:regular-noise},~\ref{assumption:heterogeneity}, when $0<\beta<1$, FedMuon in Algorithm~\ref{alg:fedmuon} can achieve the following convergence upper bound:
		\begin{align}
			\frac{1}{T} \sum_{t=0}^{T-1}\mathbb{E}[\| \nabla f(\bar{X}_{t})  \|_F ]&  \leq \frac{	f({X}_{0})  -    f({X}_{*})}{\eta T}  + \frac{\eta n L}{2}  + 4 \eta \tau n L   + 8\beta \eta \tau^2 n L   + 4\beta\tau\sqrt{n}\sigma  + 2\beta\tau\sqrt{n}\delta \notag \\
			& \quad  +    \frac{2 \sqrt{n} \sigma}{\beta T}+  \frac{2\eta n L}{\beta}  +  \frac{2 \sqrt{\beta} \sqrt{n} \sigma}{\sqrt{K}}  \ . 
		\end{align}
\end{theorem}

\begin{corollary}\label{corollary:regular-noise-t}
	In Theorem~\ref{theorem:regular-noise}, for a sufficiently large $T$,  by setting $\eta=\frac{K^{1/4}}{T^{3/4}}$,  $\beta=\frac{K^{1/2}}{T^{1/2}}$, and $\tau= \frac{T^{1/4}}{K^{3/4}}$, FedMuon in Algorithm~\ref{alg:fedmuon} can achieve the following convergence upper bound:
		\begin{align}
			&\quad \frac{1}{T} \sum_{t=0}^{T-1}\mathbb{E}[\| \nabla f(\bar{X}_{t})  \|_F ] \notag \\
			&\leq O\left( \frac{	f({X}_{0})  -    f({X}_{*}) + nL + \sqrt{n}\sigma + \sqrt{n}\delta}{(KT)^{1/4}} +  \frac{ nL+\sqrt{n} \sigma}{(KT)^{1/2}}  + \frac{nL}{(KT)^{3/4}} +\frac{K^{1/4} n L}{T^{3/4}}  \right) \ . 
		\end{align}
\end{corollary}

Since the first term in the convergence upper bound in  Corollary~\ref{corollary:regular-noise-t} dominates the other terms, FedMuon achieves a convergence rate of $O\left(\frac{1}{(KT)^{1/4}}\right)$, which  indicates a linear speedup  with respect to the number of workers $K$. This convergence rate matches that of the vector-based counterparts that also use momentum, such as  FedAvg-M \cite{cheng2024momentum}, SCAFFOLD-M \cite{cheng2024momentum}, and PAdaMFed \cite{yan2025problemparameterfree}.

\begin{remark}\label{remark:parameter-free}
	(\textbf{Problem-Parameter-Free Hyperparameters}) In Corollary~\ref{corollary:regular-noise-t}, the learning rate $\eta$, the momentum coefficient $\beta$, and the communication period $\tau$ do not depend on problem-specific parameters, such as the Lipschitz constant $L$. They depend only on the number of workers and the number of iterations, making them easy to tune. This is consistent with the vector-based method PAdaMFed \cite{yan2025problemparameterfree}. 
\end{remark}

\begin{corollary} \label{corollary:regular-noise-epsilon}
	In Theorem~\ref{theorem:regular-noise}, by setting $T=O(\frac{1}{K\epsilon^{4}})$, $\eta=O(K\epsilon^3)$, $\beta=O(K\epsilon^2)$, $\tau=O(\frac{1}{K\epsilon})$, FedMuon achieves the $\epsilon$-accuracy solution: $\frac{1}{T} \sum_{t=0}^{T-1}\mathbb{E}[\| \nabla f(\bar{X}_{t})  \|_F ]\leq O(\epsilon)$, where $\epsilon>0$ is a constant. 
\end{corollary}
From Corollary~\ref{corollary:regular-noise-epsilon}, it is easy to know that the communication complexity of FedMuon is $T/\tau=O(\frac{1}{\epsilon^3})$.  This communication complexity  is the same as $O(\frac{1}{\epsilon^3})$ of the vector-based counterparts that also use momentum, such as LocalSGDM \cite{yu2019linear}, FedAvg-M \cite{cheng2024momentum}, SCAFFOLD-M \cite{cheng2024momentum}, and PAdaMFed \cite{yan2025problemparameterfree} (see Table~\ref{table:comparison}).

\subsection{Convergence Rate Under Heavy-Tailed Noise}\label{section:heavy-tailed-noise}
In this subsection, we establish the convergence rate of FedMuon given Assumptions~\ref{assumption:smoothness},~\ref{assumption:heavy-tailed-noise},~\ref{assumption:heterogeneity}. To the best of our knowledge, this is the first work establishing the convergence rate for Muon in federated learning given  heavy-tailed noise. 

\begin{theorem} \label{theorem:heavy-tailed-noise}
	Given Assumptions~\ref{assumption:smoothness},~\ref{assumption:heavy-tailed-noise},~\ref{assumption:heterogeneity}, when $0<\beta<1$, FedMuon in Algorithm~\ref{alg:fedmuon} can achieve the following convergence upper bound:
		\begin{align}
			\frac{1}{T} \sum_{t=0}^{T-1}\mathbb{E}[\| \nabla f(\bar{X}_{t})  \|_F ]&  \leq \frac{	f({X}_{0})  -    f({X}_{*})}{\eta T}  + \frac{\eta n L}{2}  + 4 \eta \tau n L   +8 \eta\beta\tau^2  n L + 8 \sqrt{2}\beta\tau \sqrt{n}\sigma  +  2 \beta\tau \sqrt{n}\delta \notag \\
			& \quad  +   \frac{4 \sqrt{2n}\sigma}{\beta T}+  \frac{2\eta n L}{\beta}  +  \frac{4\sqrt{2n}\beta^{1-\frac{1}{p}} }{K^{1-\frac{1}{p}}}\sigma   \ . 
		\end{align}
\end{theorem}
In Theorem~\ref{theorem:heavy-tailed-noise}, the last term demonstrates how the tail index $p$ affects the convergence upper bound.  Note that \cite{sfyraki2025lions} requires the gradient clipping operation to establish the convergence rate of Muon in the single-machine setting. In contrast, our algorithm and proof do NOT require such a clipping operation.

\begin{corollary}\label{corollary:heavy-tailed-noise-t}
	In Theorem~\ref{theorem:heavy-tailed-noise}, for a sufficiently large $T$,  by setting $\eta=\frac{K^{1/4}}{T^{3/4}}$,  $\beta=\frac{K^{1/2}}{T^{1/2}}$, and $\tau=\frac{T^{1/4}}{K^{3/4}}$, FedMuon in Algorithm~\ref{alg:fedmuon} can achieve the following convergence upper bound:
		\begin{align}
			& \quad \frac{1}{T} \sum_{t=0}^{T-1}\mathbb{E}[\| \nabla f(\bar{X}_{t})  \|_F ]   \notag \\
			& \leq O\left( \frac{	f({X}_{0})  -    f({X}_{*}) + nL + \sqrt{n}\sigma + \sqrt{n}\delta}{(KT)^{1/4}} +  \frac{ nL +\sqrt{n} \sigma}{(KT)^{1/2}} + \frac{nL}{(KT)^{3/4}} +\frac{K^{1/4} n L}{T^{3/4}} + \frac{  \sqrt{n} \sigma}{(KT)^{\frac{p-1}{2p}}}  \right) \ . 
		\end{align}
\end{corollary}
\begin{remark}
	Since $p\in (1, 2]$, the last term in the convergence upper bound of Corollary~\ref{corollary:heavy-tailed-noise-t} dominates the other terms. Therefore, the convergence rate of FedMuon under heavy-tailed noise is $O\left( \frac{1}{(KT)^{\frac{p-1}{2p}}} \right)$. This convergence rate also indicates a linear speedup with respect to the number of workers, and the hyperparameter does not rely on problem-specific parameters like Liptschitz constant and gradient noise.  Moreover, when $K=1$, it matches the convergence rate of the single-machine method \cite{liu2024nonconvex}. When $p=2$, it reduces to the convergence rate in Corollary~\ref{corollary:regular-noise-t}. Furthermore,  like the regular noise setting, the learning rate $\eta$, the momentum coefficient $\beta$, and the communication period $\tau$ do not depend on problem-specific parameters, such as the Lipschitz constant $L$. They depend only on the number of workers and the number of iterations. 
\end{remark}

\begin{corollary} \label{corollary:heavy-tailed-epsilon}
	In Theorem~\ref{theorem:heavy-tailed-noise}, by setting $T=O\left(\frac{1}{K\epsilon^{\frac{2p}{p-1}}}\right)$, $\eta=O\left(K\epsilon^{\frac{3p}{2(p-1)}}\right)$, $\beta=O\left(K\epsilon^{\frac{p}{p-1}}\right)$, $\tau=O\left(\frac{1}{K\epsilon^{\frac{p}{2(p-1)}}}\right)$, FedMuon achieves the $\epsilon$-accuracy solution: $\frac{1}{T} \sum_{t=0}^{T-1}\mathbb{E}[\| \nabla f(\bar{X}_{t})  \|_F ]\leq O(\epsilon)$, where $\epsilon>0$ is a constant. 
\end{corollary}
From Corollary~\ref{corollary:heavy-tailed-epsilon}, we can know that the communication complexity is $T/\tau=O\left(\frac{1}{\epsilon^{\frac{3p}{2(p-1)}}}\right)$. When $p=2$, it reduces to the communication complexity in Corollary~\ref{corollary:regular-noise-epsilon}.

\section{Sketch of the Proof}

\subsection{Proof Sketch of Theorem~\ref{theorem:regular-noise}}
In federated learning, a key step in establishing the convergence rate is to bound the consensus error. This step typically introduces some constraints on the learning rate, such as a dependence on the Lipschitz constant. On the contrary, due to the orthonormalization operation in FedMuon, it does not introduce any constraints on the learning rate. Specifically, we have the following consensus error regarding momentum.
\begin{lemma}\label{lemma:consensus-m-main}
	(\textbf{Consensus Error w.r.t. Momentum}) 
	Given Assumptions~\ref{assumption:smoothness},~\ref{assumption:regular-noise},~\ref{assumption:heterogeneity},  the following inequality holds:    \begin{align}
		& \quad \frac{1}{K}\sum_{k=1}^{K}\mathbb{E}[\|\bar{M}_{t} -M^{(k)}_{t} \|_F] \leq 4\beta \eta \tau^2 L  \sqrt{n} + 2\beta\tau\sigma  + \beta\tau\delta \ . 
	\end{align}
\end{lemma}
It can be observed that there is no constraint on the learning rate $\eta$. In contrast, the vector-based counterpart, STEM~\cite{karimireddy2020scaffold}, imposes certain constraints on $\eta$ (see its Lemma A.9) in order to bound the consensus error related to momentum. The key reason for this difference is that the update direction $\|U^{(k)}_{t}(V^{(k)}_{t})^T\|_F$ in FedMuon is upper bounded.

\begin{lemma}\label{lemma:consensus-x-main}
	(\textbf{Consensus Error w.r.t. Variable}) Given Assumptions~\ref{assumption:smoothness},~\ref{assumption:regular-noise},~\ref{assumption:heterogeneity},  the following inequality holds:
		\begin{align}
			&  \frac{1}{K}\sum_{k=1}^{K} \| \bar{X}_{t} - X^{(k)}_{t}\|_F \leq 2\eta \tau \sqrt{n}  \ . 
		\end{align}
\end{lemma}
Similarly, because the update direction $\|U^{(k)}_{t}(V^{(k)}_{t})^T\|_F$ in FedMuon is upper bounded, the consensus error with respect to the variable $\| \bar{X}_{t} - X^{(k)}_{t}\|_F$ is upper bounded by $2\eta \tau \sqrt{n}$, rather than bounding it by $\|\bar{M}_{t} -M^{(k)}_{t} \|_F$ as existing vector-based counterparts \cite{karimireddy2020scaffold,yu2019linear,yu2019parallel},

\begin{lemma} \label{lemma:loss-fun-main}
	(\textbf{Loss Function Update})
	Given Assumptions~\ref{assumption:smoothness},~\ref{assumption:regular-noise},~\ref{assumption:heterogeneity},  the following inequality holds:
		\begin{align}
			f(\bar{X}_{t+1})&   \leq  f(\bar{X}_{t}) -\eta   \| \nabla f(\bar{X}_{t})  \|_F  + 2\eta \sqrt{n}\underbrace{\frac{1}{K}\sum_{k=1}^{K}\|\bar{M}_{t} -M^{(k)}_{t} \|_F}_{\text{consensus error}} + 2\eta \sqrt{n} L\underbrace{\frac{1}{K}\sum_{k=1}^{K} \| \bar{X}_{t} - X^{(k)}_{t}\|_F}_{\text{consensus error}} \notag \\
			& \quad  + 2\eta \sqrt{n}\underbrace{\| \frac{1}{K}\sum_{k=1}^{K} f^{(k)}(X^{(k)}_{t}) - \frac{1}{K}\sum_{k=1}^{K} M^{(k)}_{t}\|_F}_{\text{gradient estimation error}} + \frac{\eta^2 n L}{2} \ . 
		\end{align}
\end{lemma}
This lemma characterizes how the loss function value is updated in each iteration. A key difference from existing momentum-based federated optimization algorithms \cite{yu2019linear,khanduri2021stem} that do not use a global control variate is that there exists a consensus error with respect to momentum. This is caused by the orthonormalized updating direction. The detailed proof can be found in the proof of Lemma~\ref{lemma:loss-function-decrease} in the appendix. 

It is worth noting that \textbf{there is no any constraint on the learning rate $\eta$} in Lemma~\ref{lemma:loss-fun-main}.  On the contrary, most existing methods require $\eta\leq O(\frac{1}{L})$, where the Lipschitz constant $L$ is NOT easy to know. For example, considering the vector scenario and assuming the global updating direction is $\bar{q}_t$, then most momentum-based methods, such as Lemma A.5 in \cite{khanduri2021stem}, have the following inequality in their proof:
	\begin{align}\label{eq:f-decrease-vector}
		f(\bar{x}_{t+1})&   \leq  f(\bar{x}_{t}) - \frac{\eta}{2} \|\nabla f(\bar{x}_{t})\|^2  + \left(\frac{\eta^2 L}{2} - \frac{\eta}{2}  \right) \|\bar{q}_t\|^2 + \eta \|\bar{q}_t - \frac{1}{K}\sum_{k=1}^{K}\nabla f^{(k)}({x}^{(k)}_{t})\| \ . 
	\end{align}
Here, a typical operation is to let $\frac{\eta^2 L}{2} - \frac{\eta}{2} \leq -\frac{\eta}{4}$ by setting $\eta\leq \frac{1}{2L}$. Since the updating direction $\|U^{(k)}_{t}(V^{(k)}_{t})^T\|_F$ in FedMuon is upper bounded, the constraint regarding the learning rate $\eta$ can be avoided. The details can be found in the proof of Lemma~\ref{lemma:loss-function-decrease} in Appendix.

\begin{lemma} \label{lemma:grad-err-main}
	(\textbf{Gradient Error})
	Given Assumptions~\ref{assumption:smoothness},~\ref{assumption:regular-noise},  by setting $\beta<1$, the following inequality holds:
		\begin{align}
			& \frac{1}{T}\sum_{t=0}^{T-1} \mathbb{E}[\| \frac{1}{K}\sum_{k=1}^{K} f^{(k)}(X^{(k)}_{t}) - \frac{1}{K}\sum_{k=1}^{K} M^{(k)}_{t}\|_F]  \leq \frac{1}{T} \frac{\sigma}{\beta}+  \frac{\eta \sqrt{n} L}{\beta}  +  \frac{\sqrt{\beta} \sigma}{\sqrt{K}}   \ . 
		\end{align}
\end{lemma}
This lemma characterizes the gradient error. By combining the above four lemmas, we can prove Theorem~\ref{theorem:regular-noise}. The details can be found in Appendix.

\textbf{In summary, due to the orthonormalization operation, the hyperparameter of FedMuon does not rely on the problem-specific parameters such as the Lipschitz constant.}

\subsection{Proof Sketch of Theorem~\ref{theorem:heavy-tailed-noise}}
The proof of Theorem~\ref{theorem:heavy-tailed-noise} also relies on Lemma~\ref{lemma:loss-fun-main}, Lemma~\ref{lemma:consensus-x-main}. However, the consensus error regarding momentum and the gradient estimation error are affected by the heavy-tailed noise. Therefore, we establish their upper bounds given heavy-tailed noise in the following two lemmas.

\begin{lemma}
	(\textbf{Consensus Error  w.r.t. Momentum}) 
	Given Assumptions~\ref{assumption:smoothness},~\ref{assumption:heavy-tailed-noise},~\ref{assumption:heterogeneity},  the following inequality holds:
	\begin{align}
		&  \frac{1}{K}\sum_{k=1}^{K} \mathbb{E}[\|\bar{M}_{t} -M^{(k)}_{t} \|_F]  \leq 4\sqrt{2}\beta\tau\sigma + 4\eta\beta\tau^2  L\sqrt{n} +  \beta\tau \delta \ . 
	\end{align}
\end{lemma}

\begin{lemma} 
	(\textbf{Gradient Estimation Error}) 
	Given Assumptions~\ref{assumption:smoothness},~\ref{assumption:heavy-tailed-noise},~\ref{assumption:heterogeneity},  by setting $\beta<1$, the following inequality holds:
	\begin{align}
		&  \frac{1}{T}\sum_{t=0}^{T-1} \mathbb{E}[\| \frac{1}{K}\sum_{k=1}^{K} f^{(k)}(X^{(k)}_{t}) - \frac{1}{K}\sum_{k=1}^{K} M^{(k)}_{t}\|_F]   \leq \frac{1}{T} \frac{2\sqrt{2}\sigma}{\beta}+  \frac{\eta \sqrt{n} L}{\beta}  +  \frac{2\sqrt{2}\beta^{1-\frac{1}{p}} }{K^{1-\frac{1}{p}}}\sigma  \  . 
	\end{align}
\end{lemma}
It can be observed that the gradient estimation error is affected by the heavy-tailed noise. Then, by combining these two inequalities and those in Lemma~\ref{lemma:loss-function-decrease-hvt} and Lemma~\ref{lemma:consensus-error-hvt}, we can complete the proof. 

\begin{figure*}[ht]
	\begin{center}
		\centerline{\includegraphics[scale=0.6]
			{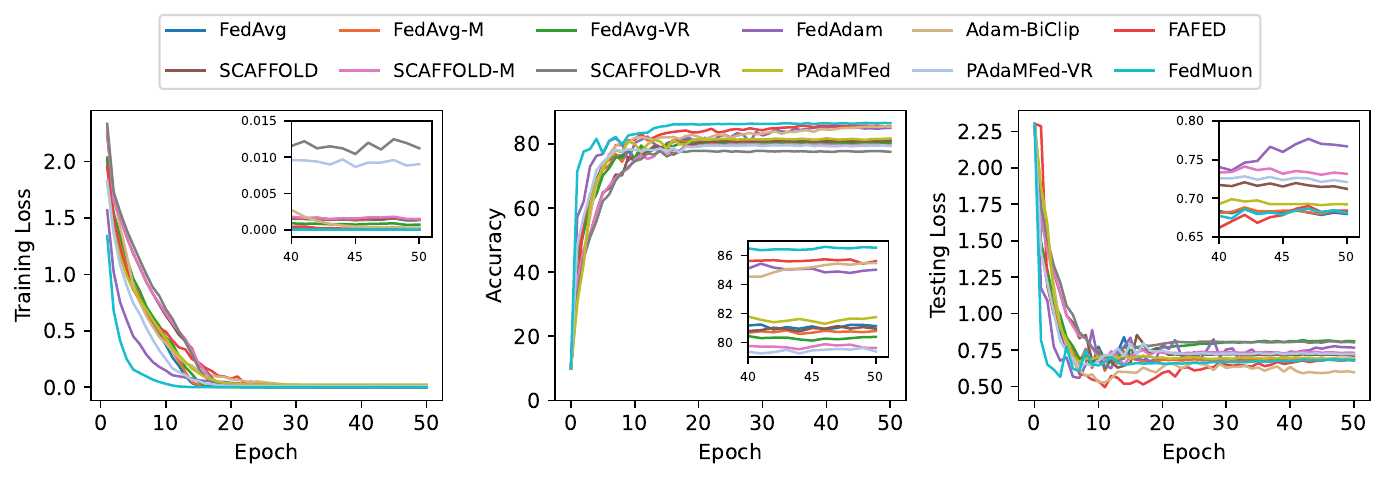}}
		\vspace{-0.2in}
		\caption{CIFAR-10 on ResNet-18 (period = 4).}
		\label{fig:cifar10_period_4}
	\end{center}
	\vspace{-0.3in}
\end{figure*}

\begin{figure*}[ht]
	\vspace{-0.1in}
	\begin{center}
		\centerline{\includegraphics[scale=0.6]
			{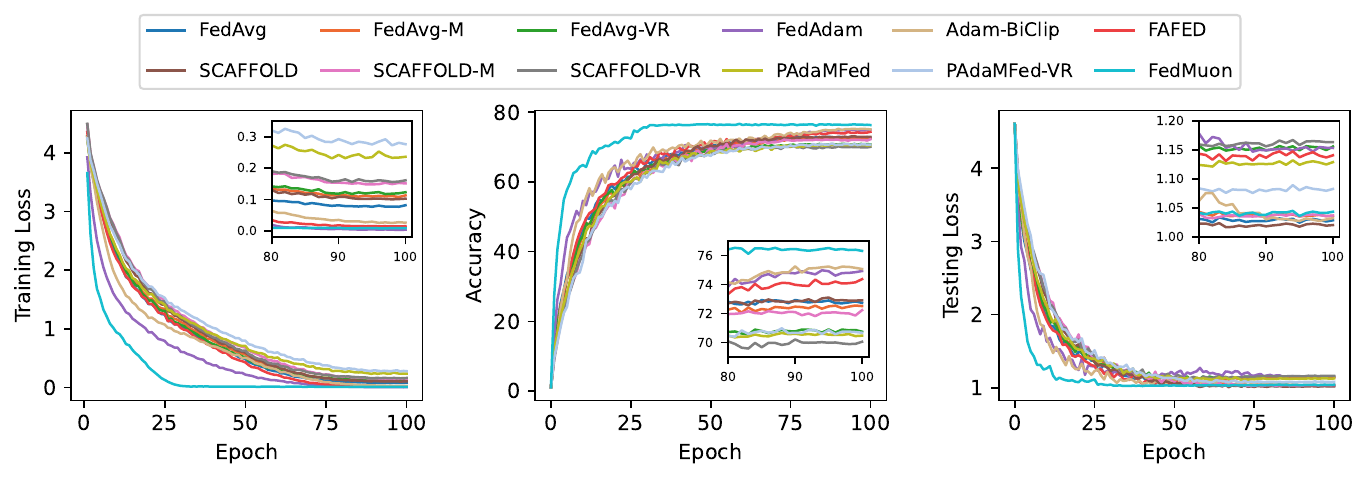}}
		\vspace{-0.2in}
		\caption{CIFAR-100 on ResNet-18 (period = 4).}
		\label{fig:cifar100_period_4}
	\end{center}
	\vspace{-0.3in}
\end{figure*}

\section{Experiments}\label{sec:exp}
In our experiments, we evaluate the performance of FedMuon on three types of deep neural networks: convolutional neural networks, recurrent neural networks, and transformers, using both image and text datasets.
\paragraph{Experiment Settings.}
In our experiments, we include four categories of baselines to provide a comprehensive comparison. Specifically, we consider 1) the classical method, FedAvg~\cite{yu2019parallel}; 2) the control-variate-based method, including SCAFFOLD~\cite{karimireddy2020scaffold}, FedAvg-M/FedAvg-VR~\cite{cheng2024momentum}, SCAFFOLD-M/SCAFFOLD-VR~\cite{cheng2024momentum}; 3) the problem-parameter-free method, including PAdaMFed/PAdaMFed-VR ~\cite{yan2025problemparameterfree}; and 4) adaptive methods, including FedAdam~\cite{reddi2020adaptive}, FAFED~\cite{wu2023faster}, and Adam-BiClip~\cite{lee2025efficient}, where the last one uses the gradient clipping method to address heavy-tailed noise.  Our federated environment is implemented on four NVIDIA RTX 6000 GPUs, where two workers are assigned to each GPU to simulate distributed clients, resulting in a total of eight workers ($K=8$) participating in the federated training.

\subsection{Image Classification with ResNet and Transformer}

We conduct experiments on two widely used image classification benchmarks, CIFAR-10 and CIFAR-100~\cite{krizhevsky2009learning}.

First, we adopt ResNet-18~\cite{he2016deep} model for image classification. For fair comparisons, the hyperparameters of all baseline algorithms are carefully tuned through grid search to ensure their best performance. In particular, for FedMuon, the learning rate $\eta$ is selected from $\{0.001, 0.002, 0.005, 0.01, 0.05 \}$, and the weight decay from $\{0.0001, 0.001, 0.01, 0.05, 0.1, 0.2 \}$. All methods are trained with a cosine decaying learning rate schedule. The momentum hyperparameters $\beta$ for all methods are fixed at 0.9. The batch size of all datasets on each worker is 64.

The training loss, test accuracy, and testing loss are presented in Figure~\ref{fig:cifar10_period_4} and Figure~\ref{fig:cifar100_period_4} with communication period set to 4. The results demonstrate taht FedMuon exhibits a substantially faster decline in training loss, indicating fast convergence and improved learning efficiency compared with the baselines. Moreover, it consistently outperforms all competing baselines and achieves the highest testing accuracy over epochs. The testing loss curves further show that our approach attains a generalization ability comparable to or better than existing methods.

To further validate the generality of our approach on modern architectures, we also consider a Vision Transformer (ViT) model~\cite{dosovitskiy2020image} without pre-training, with the results shown in Figure~\ref{fig:cifar10_vit_period_4}. It can be observed that the improvement of FedMuon over the baselines is more significant on ViT compared to ResNet-18. In particular, adaptive baselines such as FedAdam, FAFED and Adam-BiClip, can achieve better performance than other methods, which is consistent with the analysis in ~\cite{zhang2020adaptive,  kunstner2024heavy, zhang2024transformers}. Furthermore, the block heterogeneity phenomenon in Transformers identified by~\cite{zhang2024transformers} can also be effectively mitigated by FedMuon, contributing to its superior performance.

Additional results, including those with a communication period of 16, CIFAR10 under the heterogeneous settings,  CIFAR100 with the ViT model, and the text classification task, are provided in Appendix~\ref{apdx_exp}.

\begin{figure*}[ht]
	\begin{center}
		\centerline{\includegraphics[scale=0.6]
			{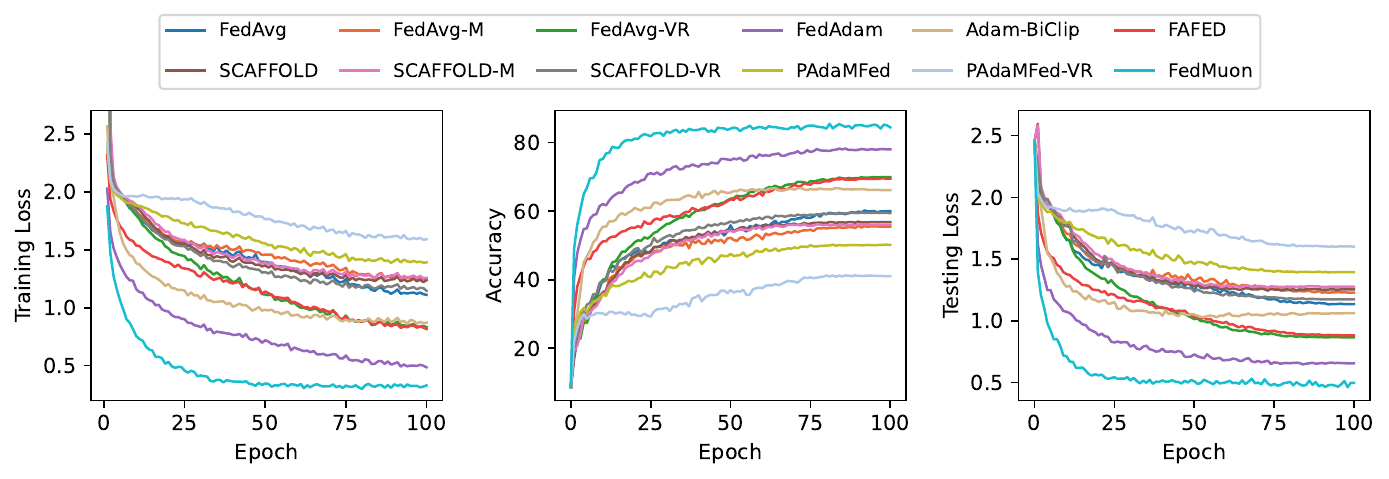}}
		\vspace{-0.2in}
		\caption{CIFAR-10 on ViT (period = 4).}
		\label{fig:cifar10_vit_period_4}
	\end{center}
	\vspace{-0.4in}
\end{figure*}

\section{Conclusion}
In this paper, we developed a novel federated learning algorithm based on Muon optimizer. Our theoretical analysis identifies multiple favorable properties of Muon in the federated learning setting. In particular, its learning rate does not require the prior knowledge regarding the problem-specific parameter, and it can naturally accommodate heavy-tailed noise. The extensive experiments confirm the efficacy of our algorithm.

\newpage
\bibliographystyle{abbrv}
\bibliography{sample-base}

\begin{thebibliography}{10}

\bibitem{an2025asgo}
K.~An, Y.~Liu, R.~Pan, Y.~Ren, S.~Ma, D.~Goldfarb, and T.~Zhang.
\newblock Asgo: Adaptive structured gradient optimization.
\newblock {\em arXiv preprint arXiv:2503.20762}, 2025.

\bibitem{bernstein2024old}
J.~Bernstein and L.~Newhouse.
\newblock Old optimizer, new norm: An anthology.
\newblock {\em arXiv preprint arXiv:2409.20325}, 2024.

\bibitem{chen2025muon}
L.~Chen, J.~Li, and Q.~Liu.
\newblock Muon optimizes under spectral norm constraints.
\newblock {\em arXiv preprint arXiv:2506.15054}, 2025.

\bibitem{cheng2024momentum}
Z.~Cheng, X.~Huang, P.~Wu, and K.~Yuan.
\newblock Momentum benefits non-iid federated learning simply and provably.
\newblock In {\em The Twelfth International Conference on Learning
  Representations}, 2024.

\bibitem{dosovitskiy2020image}
A.~Dosovitskiy, L.~Beyer, A.~Kolesnikov, D.~Weissenborn, X.~Zhai,
  T.~Unterthiner, M.~Dehghani, M.~Minderer, G.~Heigold, S.~Gelly, et~al.
\newblock An image is worth 16x16 words: Transformers for image recognition at
  scale.
\newblock {\em arXiv preprint arXiv:2010.11929}, 2020.

\bibitem{elman1990finding}
J.~L. Elman.
\newblock Finding structure in time.
\newblock {\em Cognitive science}, 14(2):179--211, 1990.

\bibitem{go2009twitter}
A.~Go, R.~Bhayani, and L.~Huang.
\newblock Twitter sentiment classification using distant supervision.
\newblock {\em CS224N project report, Stanford}, 1(12):2009, 2009.

\bibitem{he2016deep}
K.~He, X.~Zhang, S.~Ren, and J.~Sun.
\newblock Deep residual learning for image recognition.
\newblock In {\em Proceedings of the IEEE conference on computer vision and
  pattern recognition}, pages 770--778, 2016.

\bibitem{hsu2019measuring}
T.-M.~H. Hsu, H.~Qi, and M.~Brown.
\newblock Measuring the effects of non-identical data distribution for
  federated visual classification.
\newblock {\em arXiv preprint arXiv:1909.06335}, 2019.

\bibitem{jordan2024muon}
K.~Jordan, Y.~Jin, V.~Boza, J.~You, F.~Cesista, L.~Newhouse, and J.~Bernstein.
\newblock Muon: An optimizer for hidden layers in neural networks, 2024.

\bibitem{karimireddy2020scaffold}
S.~P. Karimireddy, S.~Kale, M.~Mohri, S.~Reddi, S.~Stich, and A.~T. Suresh.
\newblock Scaffold: Stochastic controlled averaging for federated learning.
\newblock In {\em International conference on machine learning}, pages
  5132--5143. PMLR, 2020.

\bibitem{khanduri2021stem}
P.~Khanduri, P.~Sharma, H.~Yang, M.~Hong, J.~Liu, K.~Rajawat, and P.~Varshney.
\newblock Stem: A stochastic two-sided momentum algorithm achieving
  near-optimal sample and communication complexities for federated learning.
\newblock {\em Advances in Neural Information Processing Systems},
  34:6050--6061, 2021.

\bibitem{kovalev2025understanding}
D.~Kovalev.
\newblock Understanding gradient orthogonalization for deep learning via
  non-euclidean trust-region optimization.
\newblock {\em arXiv preprint arXiv:2503.12645}, 2025.

\bibitem{krizhevsky2009learning}
A.~Krizhevsky, G.~Hinton, et~al.
\newblock Learning multiple layers of features from tiny images.
\newblock 2009.

\bibitem{kunstner2024heavy}
F.~Kunstner, A.~Milligan, R.~Yadav, M.~Schmidt, and A.~Bietti.
\newblock Heavy-tailed class imbalance and why adam outperforms gradient
  descent on language models.
\newblock {\em Advances in Neural Information Processing Systems},
  37:30106--30148, 2024.

\bibitem{lee2025efficient}
S.~H. Lee, M.~Zaheer, and T.~Li.
\newblock Efficient distributed optimization under heavy-tailed noise.
\newblock {\em arXiv preprint arXiv:2502.04164}, 2025.

\bibitem{li2025note}
J.~Li and M.~Hong.
\newblock A note on the convergence of muon and further.
\newblock {\em arXiv e-prints}, pages arXiv--2502, 2025.

\bibitem{liu2025muon}
J.~Liu, J.~Su, X.~Yao, Z.~Jiang, G.~Lai, Y.~Du, Y.~Qin, W.~Xu, E.~Lu, J.~Yan,
  et~al.
\newblock Muon is scalable for llm training.
\newblock {\em arXiv preprint arXiv:2502.16982}, 2025.

\bibitem{liu2024nonconvex}
Z.~Liu and Z.~Zhou.
\newblock Nonconvex stochastic optimization under heavy-tailed noises: Optimal
  convergence without gradient clipping.
\newblock {\em arXiv preprint arXiv:2412.19529}, 2024.

\bibitem{mcmahan2017communication}
B.~McMahan, E.~Moore, D.~Ramage, S.~Hampson, and B.~A. y~Arcas.
\newblock Communication-efficient learning of deep networks from decentralized
  data.
\newblock In {\em Artificial intelligence and statistics}, pages 1273--1282.
  PMLR, 2017.

\bibitem{pethick2025training}
T.~Pethick, W.~Xie, K.~Antonakopoulos, Z.~Zhu, A.~Silveti-Falls, and V.~Cevher.
\newblock Training deep learning models with norm-constrained lmos.
\newblock {\em arXiv preprint arXiv:2502.07529}, 2025.

\bibitem{reddi2020adaptive}
S.~J. Reddi, Z.~Charles, M.~Zaheer, Z.~Garrett, K.~Rush, J.~Kone{\v{c}}n{\`y},
  S.~Kumar, and H.~B. McMahan.
\newblock Adaptive federated optimization.
\newblock In {\em International Conference on Learning Representations}, 2020.

\bibitem{riabinin2025gluon}
A.~Riabinin, E.~Shulgin, K.~Gruntkowska, and P.~Richt{\'a}rik.
\newblock Gluon: Making muon \& scion great again!(bridging theory and practice
  of lmo-based optimizers for llms).
\newblock {\em arXiv preprint arXiv:2505.13416}, 2025.

\bibitem{sato2025analysis}
N.~Sato, H.~Naganuma, and H.~Iiduka.
\newblock Analysis of muon's convergence and critical batch size.
\newblock {\em arXiv preprint arXiv:2507.01598}, 2025.

\bibitem{sfyraki2025lions}
M.-E. Sfyraki and J.-K. Wang.
\newblock Lions and muons: Optimization via stochastic frank-wolfe.
\newblock {\em arXiv preprint arXiv:2506.04192}, 2025.

\bibitem{shen2025convergence}
W.~Shen, R.~Huang, M.~Huang, C.~Shen, and J.~Zhang.
\newblock On the convergence analysis of muon.
\newblock {\em arXiv preprint arXiv:2505.23737}, 2025.

\bibitem{stich2018local}
S.~U. Stich.
\newblock Local sgd converges fast and communicates little.
\newblock {\em arXiv preprint arXiv:1805.09767}, 2018.

\bibitem{wu2023faster}
X.~Wu, F.~Huang, Z.~Hu, and H.~Huang.
\newblock Faster adaptive federated learning.
\newblock In {\em Proceedings of the AAAI conference on artificial
  intelligence}, volume~37, pages 10379--10387, 2023.

\bibitem{yan2025problemparameterfree}
W.~Yan, K.~Zhang, X.~Wang, and X.~Cao.
\newblock Problem-parameter-free federated learning.
\newblock In {\em The Thirteenth International Conference on Learning
  Representations}, 2025.

\bibitem{yang2021achieving}
H.~Yang, M.~Fang, and J.~Liu.
\newblock Achieving linear speedup with partial worker participation in non-iid
  federated learning.
\newblock {\em arXiv preprint arXiv:2101.11203}, 2021.

\bibitem{yoshioka2024visiontransformers}
K.~Yoshioka.
\newblock vision-transformers-cifar10: Training vision transformers (vit) and
  related models on cifar-10.
\newblock \url{https://github.com/kentaroy47/vision-transformers-cifar10},
  2024.

\bibitem{yu2019linear}
H.~Yu, R.~Jin, and S.~Yang.
\newblock On the linear speedup analysis of communication efficient momentum
  sgd for distributed non-convex optimization.
\newblock In {\em International Conference on Machine Learning}, pages
  7184--7193. PMLR, 2019.

\bibitem{yu2019parallel}
H.~Yu, S.~Yang, and S.~Zhu.
\newblock Parallel restarted sgd with faster convergence and less
  communication: Demystifying why model averaging works for deep learning.
\newblock In {\em Proceedings of the AAAI conference on artificial
  intelligence}, volume~33, pages 5693--5700, 2019.

\bibitem{zhang2020adaptive}
J.~Zhang, S.~P. Karimireddy, A.~Veit, S.~Kim, S.~Reddi, S.~Kumar, and S.~Sra.
\newblock Why are adaptive methods good for attention models?
\newblock {\em Advances in Neural Information Processing Systems},
  33:15383--15393, 2020.

\bibitem{zhang2025adagrad}
M.~Zhang, Y.~Liu, and H.~Schaeffer.
\newblock Adagrad meets muon: Adaptive stepsizes for orthogonal updates.
\newblock {\em arXiv preprint arXiv:2509.02981}, 2025.

\bibitem{zhang2024transformers}
Y.~Zhang, C.~Chen, T.~Ding, Z.~Li, R.~Sun, and Z.~Luo.
\newblock Why transformers need adam: A hessian perspective.
\newblock {\em Advances in neural information processing systems},
  37:131786--131823, 2024.

\end{thebibliography}

\newpage
\appendix
\section{Additional Experiments}\label{apdx_exp}

\subsection{More Experiments about Image Classification  with ResNet and Transformer}

\begin{figure*}[h]
\vspace{-0.1in}
\begin{center}
\centerline{\includegraphics[scale=0.58]
{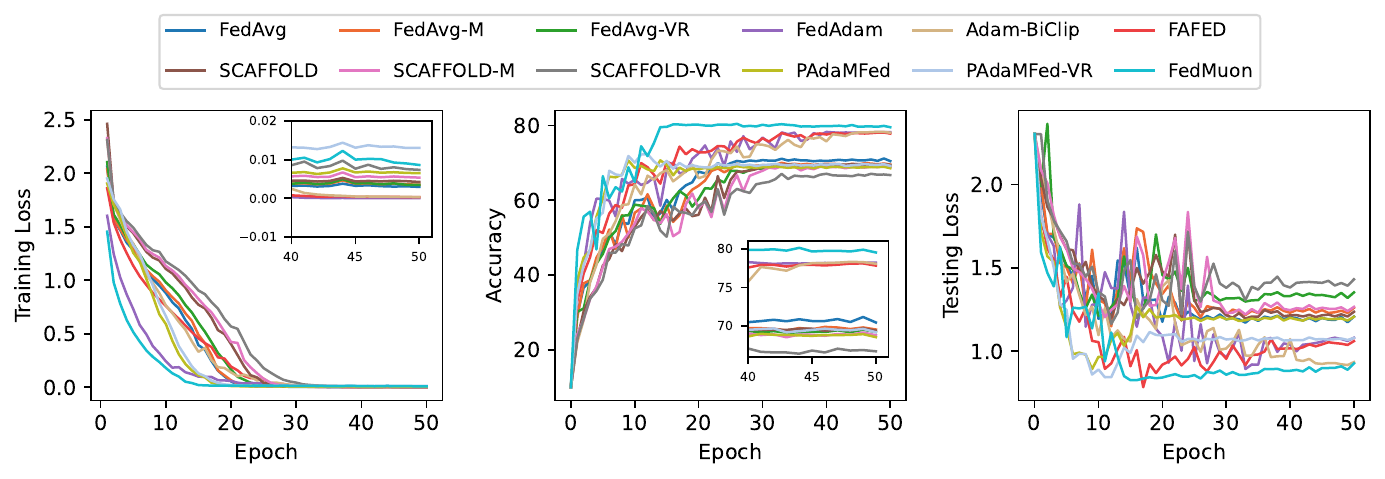}}
\vspace{-0.1in}
\caption{CIFAR-10 on ResNet-18 (period = 4, $Dir(0.5)$).}
\label{fig:cifar10_dirich_period_4}
\end{center}
\vspace{-0.1in}
\end{figure*}

To evaluate the performance on heterogeneous setting, we further conduct experiments on CIFAR10 with ResNet-18. Specifically, the data are partitioned across clients using a Dirichlet distribution~\cite{hsu2019measuring} with $Dir(0.5)$. The results are shown in Figure~\ref{fig:cifar10_dirich_period_4}. It can be observed that FedMuon consistently outperforms all baselines. While methods such as SCAFFOLD, SCAFFOLD-M/SCAFFOLD-VR, and PAdaMFed/PAdaMFed-VR mitigate date heterogeneity by introducing additional control variate, FedMuon stilll achieves superior performance without employing any such auxiliary mechanism, demonstrating its simplicity and effectiveness in heterogeneous environments.

\begin{figure*}[h]
\begin{center}
\centerline{\includegraphics[scale=0.58]
{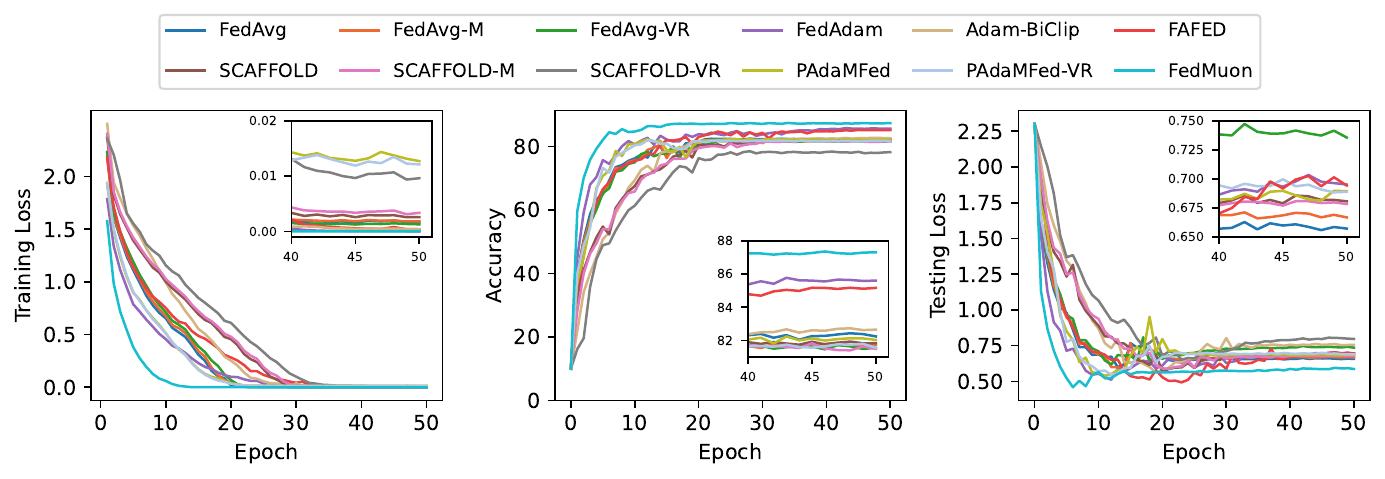}}
\vspace{-0.1in}
\caption{CIFAR-10 on ResNet-18 (period = 16).}
\label{fig:cifar10_period_16}
\end{center}
\vspace{-0.2in}
\end{figure*}

\begin{figure*}[h]
\begin{center}
\centerline{\includegraphics[scale=0.58]
{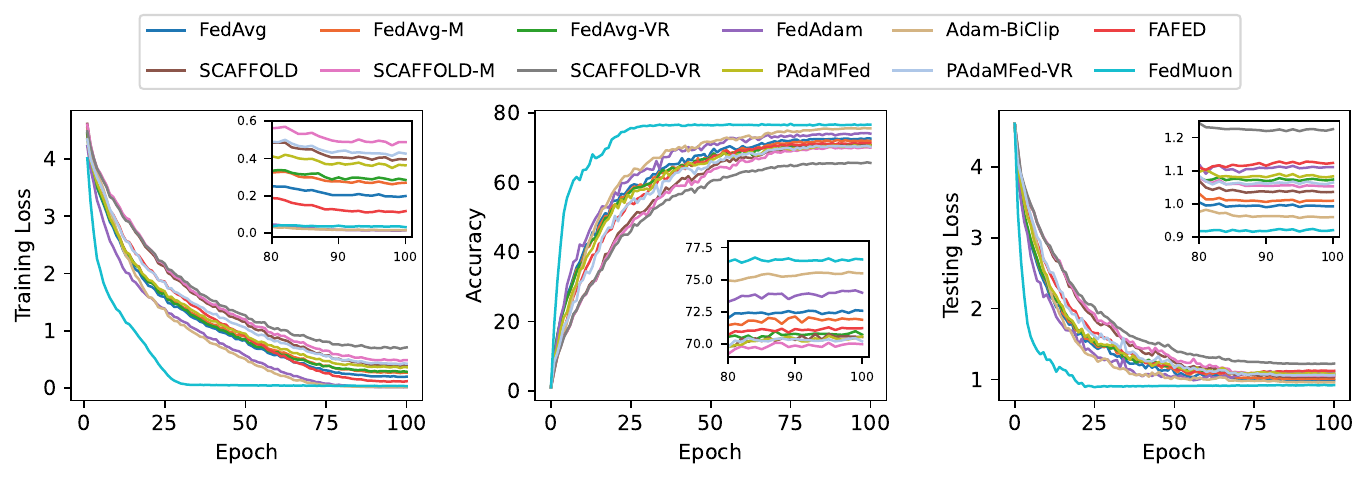}}
\vspace{-0.1in}
\caption{CIFAR-100 on ResNet-18 (period = 16).}
\label{fig:cifar100_period_16}
\end{center}
\vspace{-0.2in}
\end{figure*}

Moreover, for the homogeneous setting, we also report the results with a communication period of 16, as shown in Figure~\ref{fig:cifar10_period_16} and Figure~\ref{fig:cifar100_period_16}. The results show that FedMuon continues to outperform the baselines even when the communication period is high, and the performance gap becomes larger at higher communication periods, particularly in terms of testing loss.

For the ViT model, we follow the implementation  of ~\cite{yoshioka2024visiontransformers} and summarize the detail of architecture settings in Table~\ref{tab:vit_model}.

\begin{table}[h]
\centering
\caption{Architecture of the Vision Transformer (ViT) model.}
\label{tab:vit_model}
\vspace{0.1in}
\begin{tabular}{l c}
\hline
\textbf{Component} & \textbf{Configuration} \\
\hline
Image patches (batches)      & 4 \\
Attention heads per layer    & 8 \\
Dimension per head           & 64 \\
Transformer encoder depth    & 6 \\
Dropout rate of encoder       & 0.1 \\
MLP dimension                & 512 \\
Dropout rate of MLP           & 0.1 \\
\hline
\end{tabular}
\end{table}

More experimental results on CIFAR-10 with the ViT model are presented in Figure~\ref{fig:cifar10_vit_period_16}, and the results on CIFAR-100 are shown in Figure~\ref{fig:cifar100_vit_period_4} and Figure~\ref{fig:cifar100_vit_period_16}. These results further confirm the effectiveness of FedMuon on Transformer-based architectures.

\begin{figure*}[h]
\vspace{-0.1in}
\begin{center}
\centerline{\includegraphics[scale=0.58]
{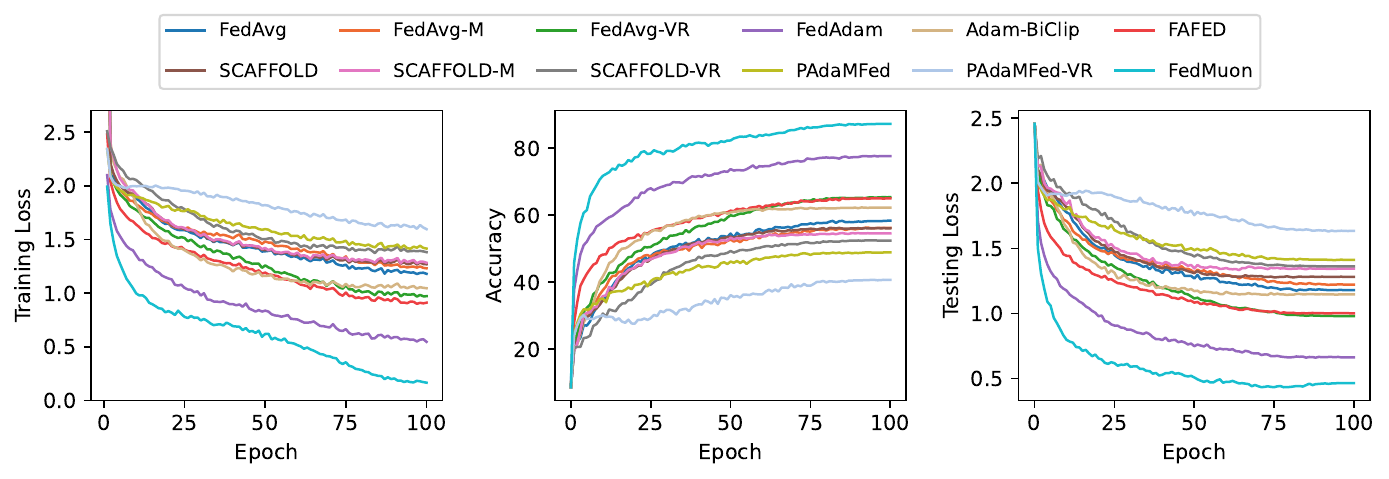}}
\vspace{-0.2in}
\caption{CIFAR-10 on ViT (period = 16).}
\label{fig:cifar10_vit_period_16}
\end{center}
\vspace{-0.2in}
\end{figure*}

\begin{figure*}[h]
\begin{center}
\centerline{\includegraphics[scale=0.58]
{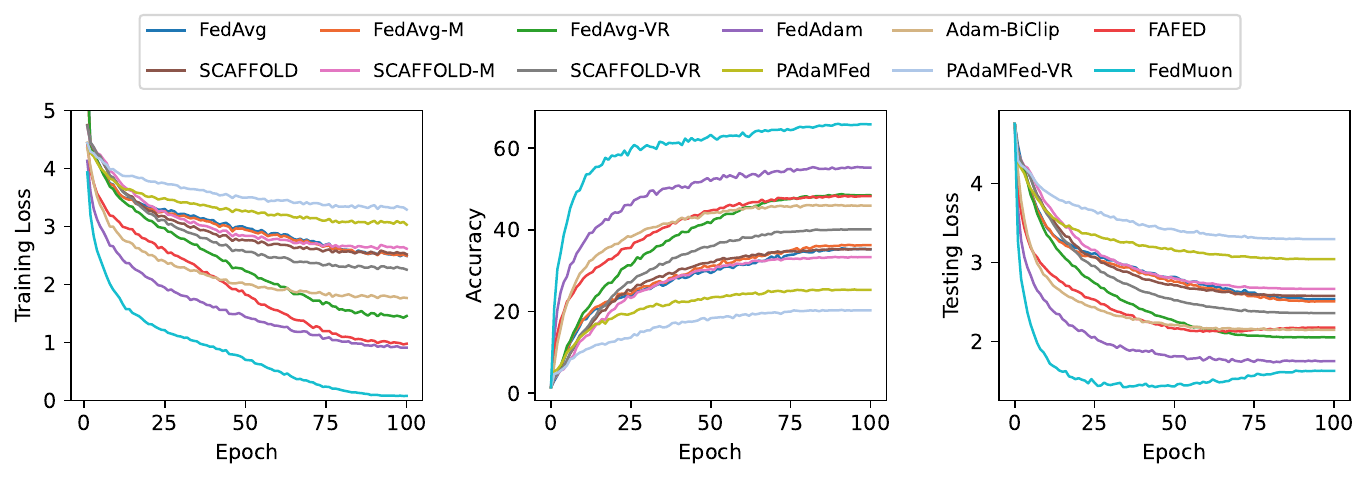}}
\vspace{-0.1in}
\caption{CIFAR-100 on ViT (period = 4).}
\label{fig:cifar100_vit_period_4}
\end{center}
\vspace{-0.2in}
\end{figure*}

\begin{figure*}[h]
\begin{center}
\centerline{\includegraphics[scale=0.58]
{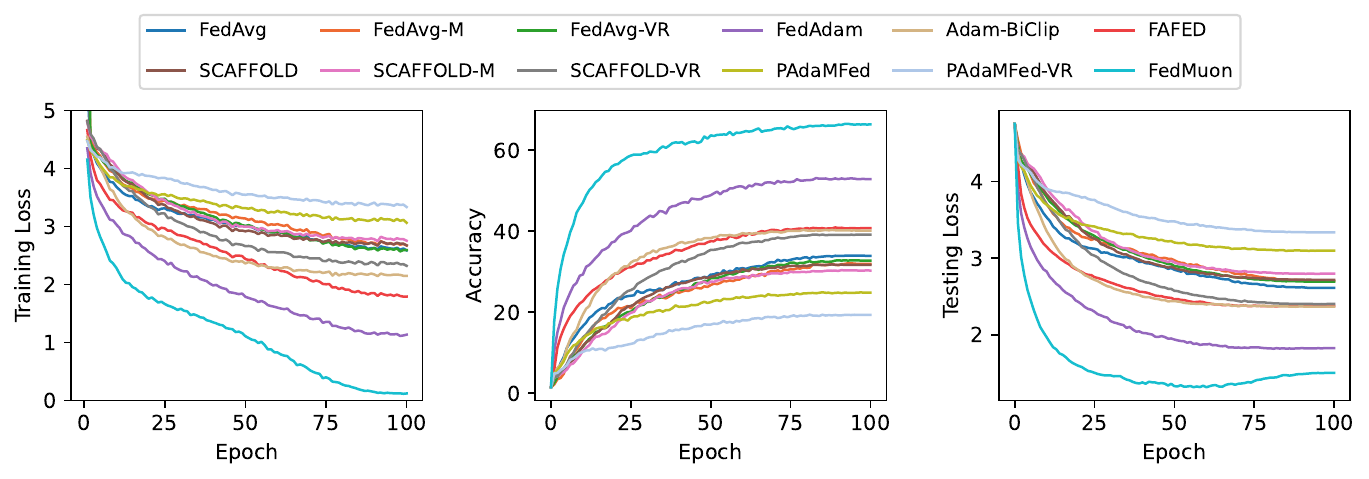}}
\vspace{-0.1in}
\caption{CIFAR-100 on ViT (period = 16).}
\label{fig:cifar100_vit_period_16}
\end{center}
\vspace{-0.2in}
\end{figure*}

\newpage

\subsection{Text Classification with RNN}
Next, we evaluate our approach on a text classification task,  where the data naturally exhibit heavy-tailed noise characteristics. We use the Sentiment140 dataset~\cite{go2009twitter} and adopt a recurrent neural network (RNN)~\cite{elman1990finding}. For the Sentiment140 dataset, the original corpus contains 1.6 million training samples and a testing set of merely 498 samples. To avoid overly fast convergence and better observe the training dynamics, we randomly subsample the training set and retain only 1\% of the original training data for model training. The batch size of  Sentiment140 dataset on each worker is 64. For the RNN model used in text classification task, we summarize the detail of architecture settings in Table~\ref{tab:rnn_moel}. 
\begin{table}[h]
\centering
\caption{Architecture of the RNN model.}
\label{tab:rnn_moel}
\vspace{0.1in}
\begin{tabular}{l c}
\hline
\textbf{Component} & \textbf{Dimension} \\
\hline
Input dimension  & 300 \\
Hidden dimension & 4096 \\
Output dimension & 2 \\
\hline
\end{tabular}
\end{table}

The results are presented in Figure~\ref{fig:sentiment140_period_4}. FedMuon consistently outperforms the baselines across all metrics. Moreover, it can be observed that adaptive methods, such as FedAdam, FAFED, and Adam-BiClip, demonstrate greater robustness to heavy-tailed noise compared with other approaches, which is consistent with prior findings~\cite{zhang2020adaptive, kunstner2024heavy}. This observation aligns with explanations that in language tasks, the heavy-tailed class imbalance causes infrequent words to converge more slowly under gradient descent, whereas adaptive methods are less sensitive to this issue~\cite{kunstner2024heavy}. FedMuon similarly benefits from this robustness, which explains its strong performance under heavy-tailed settings. 


\begin{figure*}[h]
\begin{center}
\centerline{\includegraphics[scale=0.6]
{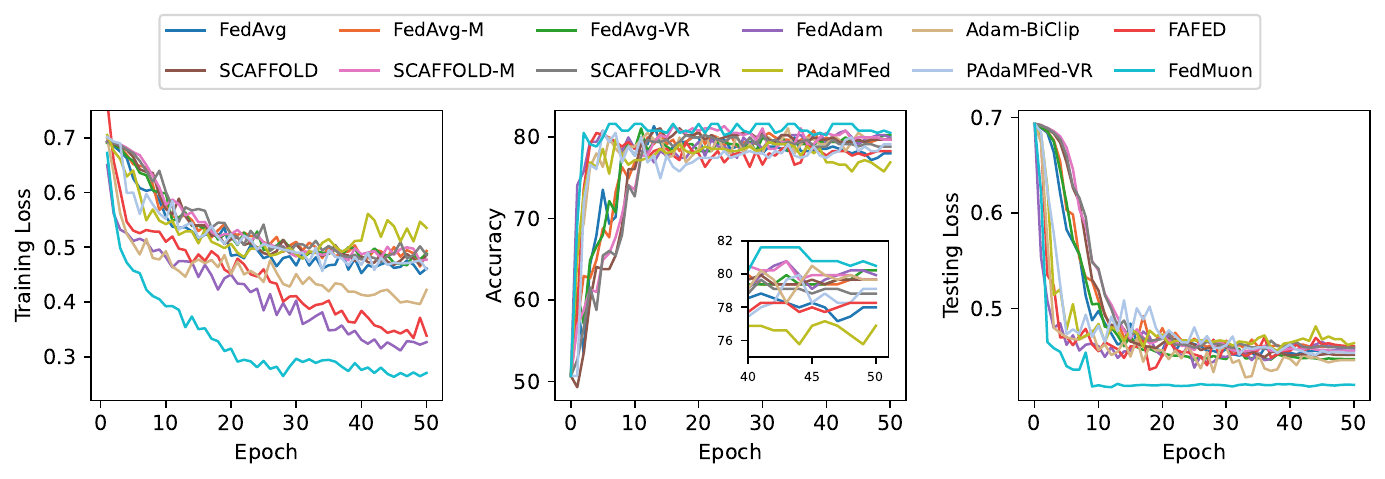}}
\vspace{-0.2in}
\caption{Sentiment140 on RNN (period = 4).}
\label{fig:sentiment140_period_4}
\end{center}
\vspace{-0.3in}
\end{figure*}


More experimental results with communication period of 16 is shown in Figure~\ref{fig:sentiment140_period_16}. It can be observed that when the communication period is increased to 16, FedMuon achieves even larger performance gains over the baselines compared with the case of period 4. This demonstrates that FedMuon is particularly effective under infrequent communication, as it benefits from the reduced communication complexity while still maintaining superior performance.

\begin{figure*}[h]
\begin{center}
\centerline{\includegraphics[scale=0.58]
{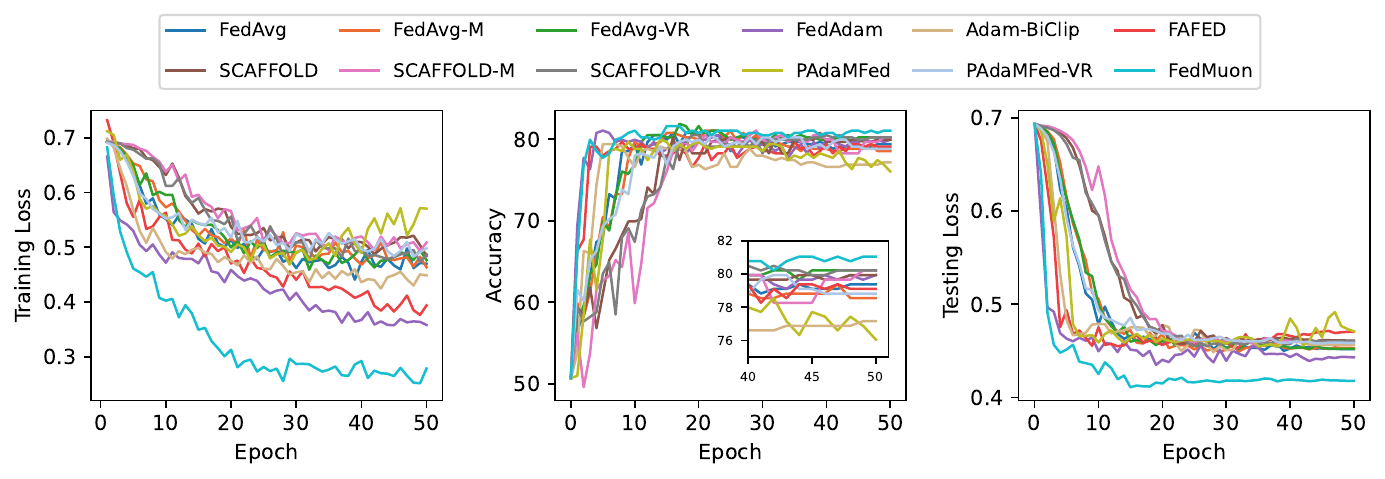}}
\vspace{-0.1in}
\caption{Sentiment140 on RNN (period = 16).}
\label{fig:sentiment140_period_16}
\end{center}
\vspace{-0.2in}
\end{figure*}
\newpage
\section{Convergence Analysis Under Regular Noise}\label{sec:conv_reg_noise}

\begin{lemma}\label{lemma:loss-function-decrease}
	Given Assumptions~\ref{assumption:smoothness},~\ref{assumption:regular-noise},~\ref{assumption:heterogeneity},  the following inequality holds:
	\begin{align}
		f(\bar{X}_{t+1})&   \leq  f(\bar{X}_{t}) -\eta   \| \nabla f(\bar{X}_{t})  \|_F  + 2\eta \sqrt{n}\frac{1}{K}\sum_{k=1}^{K}\|\bar{M}_{t} -M^{(k)}_{t} \|_F + \frac{\eta^2 n L}{2}\notag \\
        & \quad + 2\eta \sqrt{n} L\frac{1}{K}\sum_{k=1}^{K} \| \bar{X}_{t} - X^{(k)}_{t}\|_F + 2\eta \sqrt{n}\| \frac{1}{K}\sum_{k=1}^{K} f^{(k)}(X^{(k)}_{t}) - \frac{1}{K}\sum_{k=1}^{K} M^{(k)}_{t}\|_F \ . 
	\end{align}
\end{lemma}

\begin{proof}

    \begin{align}
		f(\bar{X}_{t+1})&  \leq f(\bar{X}_{t}) + \langle \nabla f(\bar{X}_{t}) , \bar{X}_{t+1}-\bar{X}_{t}\rangle + \frac{L}{2}\|\bar{X}_{t+1}-\bar{X}_{t}\|_F^2 \notag \\
		& \leq f(\bar{X}_{t}) -\eta \langle \nabla f(\bar{X}_{t}) , \frac{1}{K}\sum_{k=1}^{K}U^{(k)}_{t}(V^{(k)}_{t})^T\rangle + \frac{\eta^2 L}{2}\|\frac{1}{K}\sum_{k=1}^{K}U^{(k)}_{t}(V^{(k)}_{t})^T\|_F^2 \notag \\
		& = f(\bar{X}_{t}) -\eta \frac{1}{K}\sum_{k=1}^{K} \langle \nabla f(\bar{X}_{t}) -\bar{M}_{t} , U^{(k)}_{t}(V^{(k)}_{t})^T \rangle -\eta \frac{1}{K}\sum_{k=1}^{K} \langle \bar{M}_{t} -M^{(k)}_{t} , U^{(k)}_{t}(V^{(k)}_{t})^T \rangle \notag \\
        & \quad -\eta  \frac{1}{K}\sum_{k=1}^{K} \langle M^{(k)}_{t}, U^{(k)}_{t}(V^{(k)}_{t})^T \rangle  + \frac{\eta^2 L}{2}\|\frac{1}{K}\sum_{k=1}^{K}U^{(k)}_{t}(V^{(k)}_{t})^T\|_F^2 \notag \\
		& \leq  f(\bar{X}_{t}) -\eta   \frac{1}{K}\sum_{k=1}^{K} \|  {M}^{(k)}_{t}\|_* + \eta \sqrt{n}\| \nabla f(\bar{X}_{t}) -\bar{M}_{t}\|_F + \eta \sqrt{n}\frac{1}{K}\sum_{k=1}^{K}\|\bar{M}_{t} -M^{(k)}_{t} \|_F + \frac{\eta^2 n L}{2}\notag \\
        & \leq  f(\bar{X}_{t}) -\eta   \| \nabla f(\bar{X}_{t})  \|_F + 2\eta \sqrt{n}\| \nabla f(\bar{X}_{t}) -\bar{M}_{t}\|_F + 2\eta \sqrt{n}\frac{1}{K}\sum_{k=1}^{K}\|\bar{M}_{t} -M^{(k)}_{t} \|_F + \frac{\eta^2 n L}{2}\notag \\
        & \leq  f(\bar{X}_{t}) -\eta   \| \nabla f(\bar{X}_{t})  \|_F  + 2\eta \sqrt{n}\frac{1}{K}\sum_{k=1}^{K}\|\bar{M}_{t} -M^{(k)}_{t} \|_F + \frac{\eta^2 n L}{2}\notag \\
        & \quad + 2\eta \sqrt{n}\| \nabla f(\bar{X}_{t}) -\frac{1}{K}\sum_{k=1}^{K} \nabla f^{(k)}(X^{(k)}_{t})\|_F + 2\eta \sqrt{n}\| \frac{1}{K}\sum_{k=1}^{K} f^{(k)}(X^{(k)}_{t}) - \frac{1}{K}\sum_{k=1}^{K} M^{(k)}_{t}\|_F \notag \\
        & \leq  f(\bar{X}_{t}) -\eta   \| \nabla f(\bar{X}_{t})  \|_F  + 2\eta \sqrt{n}\frac{1}{K}\sum_{k=1}^{K}\|\bar{M}_{t} -M^{(k)}_{t} \|_F + \frac{\eta^2 n L}{2}\notag \\
        & \quad + 2\eta \sqrt{n} L\frac{1}{K}\sum_{k=1}^{K} \| \bar{X}_{t} - X^{(k)}_{t}\|_F + 2\eta \sqrt{n}\| \frac{1}{K}\sum_{k=1}^{K} f^{(k)}(X^{(k)}_{t}) - \frac{1}{K}\sum_{k=1}^{K} M^{(k)}_{t}\|_F \ , 
	\end{align}
    where the fourth step holds due to $\langle  {M}^{(k)}_{t},U^{(k)}_{t}(V^{(k)}_{t})^T \rangle=\|{M}^{(k)}_{t}\|_{*}$, $-\langle \nabla f(\bar{X}_{t}) -\bar{M}_{t} , U^{(k)}_{t}(V^{(k)}_{t})^T \rangle\leq  \| \nabla f(\bar{X}_{t}) -\bar{M}_{t}\|_F\| U^{(k)}_{t}(V^{(k)}_{t})^T \|_F \leq \sqrt{n}  \| \nabla f(\bar{X}_{t}) -\bar{M}_{t}\|_F$, $-\langle \bar{M}_{t} - M^{(k)}_{t}, U^{(k)}_{t}(V^{(k)}_{t})^T \rangle\leq \| \bar{M}_{t} - M^{(k)}_{t}\|_F \|U^{(k)}_{t}(V^{(k)}_{t})^T\|_F\leq \sqrt{n}\| \bar{M}_{t} - M^{(k)}_{t}\|_F$,  and $\| U^{(k)}_{t}(V^{(k)}_{t})^T \|^2_F\leq n$, the fifth step holds due to
	$ \|  \bar{M}_{t}\|_* \leq  \|  \bar{M}_{t} -{M}^{(k)}_{t}\|_* + \|  {M}^{(k)}_{t}\|_*\leq \sqrt{n} \|  \bar{M}_{t} -{M}^{(k)}_{t}\|_F + \|  {M}^{(k)}_{t}\|_*$ and $\| \nabla f(\bar{X}_{t})  \|_F \leq \| \nabla f(\bar{X}_{t})  \|_*=\|\nabla f(\bar{X}_{t})  -\bar{M}_{t} + \bar{M}_{t} \|_* \leq \|\nabla f(\bar{X}_{t})  -\bar{M}_{t} \|_* +\| \bar{M}_{t}\|_* \leq \sqrt{n}\|\nabla f(\bar{X}_{t})  -\bar{M}_{t} \|_F +\| \bar{M}_{t}\|_*$, and the last step holds due to Assumption~\ref{assumption:smoothness}.


\end{proof}

\begin{lemma} \label{lemma:consensus-error}
	Given Assumptions~\ref{assumption:smoothness},~\ref{assumption:regular-noise},~\ref{assumption:heterogeneity},  the following inequality holds:
	\begin{align}
		&  \frac{1}{K}\sum_{k=1}^{K} \| \bar{X}_{t} - X^{(k)}_{t}\|_F \leq 2\eta \tau \sqrt{n}  \ . 
	\end{align}
\end{lemma}

\begin{proof}
	\begin{align}
		& \quad \frac{1}{K}\sum_{k=1}^{K} \| \bar{X}_{t} - X^{(k)}_{t}\|_F \notag \\
		& \leq \frac{1}{K}\sum_{k=1}^{K} \| \bar{X}_{s_t\tau} - \eta \sum_{t'=s_t\tau}^{t-1}\frac{1}{K}\sum_{k'=1}^{K}U^{(k')}_{t'}(V^{(k')}_{t'})^T - X^{(k)}_{s_t\tau} + \eta  \sum_{t'=s_t\tau}^{t-1}U^{(k)}_{t'}(V^{(k)}_{t'})^T  \|_F \notag \\
		& \leq \eta \frac{1}{K}\sum_{k=1}^{K} \|    \sum_{t'=s_t\tau}^{t-1}U^{(k)}_{t'}(V^{(k)}_{t'})^T -  \sum_{t'=s_t\tau}^{t-1}\frac{1}{K}\sum_{k'=1}^{K}U^{(k')}_{t'}(V^{(k')}_{t'})^T \|_F \notag \\
		& \leq \eta \frac{1}{K}\sum_{k=1}^{K} \|    \sum_{t'=s_t\tau}^{t-1}U^{(k)}_{t'}(V^{(k)}_{t'})^T\|_F +  \eta \frac{1}{K}\sum_{k=1}^{K} \|    \sum_{t'=s_t\tau}^{t-1}\frac{1}{K}\sum_{k'=1}^{K}U^{(k')}_{t'}(V^{(k')}_{t'})^T \|_F \notag \\
		& \leq 2\eta \tau \sqrt{n}  \ , 
	\end{align}
	where $s_t = \lfloor t/\tau\rfloor$, the last step holds due to $\| U^{(k)}_{t}(V^{(k)}_{t})^T \|_F\leq \sqrt{n}$.
\end{proof}

\begin{lemma}\label{lemma:consensus-error-momentum}
Given Assumptions~\ref{assumption:smoothness},~\ref{assumption:regular-noise},~\ref{assumption:heterogeneity},  the following inequality holds:
    \begin{align}
        & \quad \frac{1}{K}\sum_{k=1}^{K}\mathbb{E}[\|\bar{M}_{t} -M^{(k)}_{t} \|_F] \leq 4\beta \eta \tau^2 L  \sqrt{n} + 2\beta\tau\sigma  + \beta\tau\delta \ . 
    \end{align}
\end{lemma}

\begin{proof}

 \begin{align}
        & \quad \frac{1}{K}\sum_{k=1}^{K}\mathbb{E}[\|\bar{M}_{t} -M^{(k)}_{t} \|_F] \notag \\
        & = \frac{1}{K}\sum_{k=1}^{K}\mathbb{E}[\|(1-\beta)\frac{1}{K}\sum_{k'=1}^{K}M^{(k')}_{t-1} + \beta\frac{1}{K}\sum_{k'=1}^{K}\nabla f^{(k')}(X_{t}^{(k')}; \xi_{t}^{(k')})  \notag \\
        & \quad -(1-\beta) M^{(k)}_{t-1} - \beta \nabla f^{(k)}(X_{t}^{(k)}; \xi_{t}^{(k)}) \|_F] \notag \\
        & = \frac{1}{K}\sum_{k=1}^{K}\mathbb{E}[\|(1-\beta)\frac{1}{K}\sum_{k'=1}^{K}M^{(k')}_{t-1} -(1-\beta) M^{(k)}_{t-1}   \notag \\
        & \quad + \beta\frac{1}{K}\sum_{k'=1}^{K}\nabla f^{(k')}(X_{t}^{(k')}; \xi_{t}^{(k')}) - \beta\frac{1}{K}\sum_{k'=1}^{K}\nabla f^{(k')}(X_{t}^{(k')}) + \beta\frac{1}{K}\sum_{k'=1}^{K}\nabla f^{(k')}(X_{t}^{(k')}) \notag \\
        & \quad - \beta\frac{1}{K}\sum_{k'=1}^{K}\nabla f^{(k')}(\bar{X}_{t})+ \beta\frac{1}{K}\sum_{k'=1}^{K}\nabla f^{(k')}(\bar{X}_{t}) - \beta\nabla f^{(k)}(\bar{X}_{t}) \notag \\
        & \quad + \beta\nabla f^{(k)}(\bar{X}_{t})  - \beta \nabla f^{(k)}(X_{t}^{(k)}) + \beta \nabla f^{(k)}(X_{t}^{(k)})- \beta \nabla f^{(k)}(X_{t}^{(k)}; \xi_{t}^{(k)}) \|_F] \notag \\
        & \leq (1-\beta) \frac{1}{K}\sum_{k=1}^{K}\mathbb{E}[\|\bar{M}_{t-1} - M^{(k)}_{t-1}  \|_F]  + \beta \frac{1}{K}\sum_{k=1}^{K}\mathbb{E}[\|\frac{1}{K}\sum_{k'=1}^{K}\nabla f^{(k')}(X_{t}^{(k')}; \xi_{t}^{(k')}) - \frac{1}{K}\sum_{k'=1}^{K}\nabla f^{(k')}(X_{t}^{(k')})\|_F]\notag \\
        & \quad  + \beta\frac{1}{K}\sum_{k=1}^{K}\mathbb{E}[\|\frac{1}{K}\sum_{k'=1}^{K}\nabla f^{(k')}(X_{t}^{(k')}) - \frac{1}{K}\sum_{k'=1}^{K}\nabla f^{(k')}(\bar{X}_{t})\|_F]  + \beta\frac{1}{K}\sum_{k=1}^{K}\mathbb{E}[\|\frac{1}{K}\sum_{k'=1}^{K}\nabla f^{(k')}(\bar{X}_{t}) - \nabla f^{(k)}(\bar{X}_{t})\|_F] \notag \\
        & \quad + \beta\frac{1}{K}\sum_{k=1}^{K}\mathbb{E}[\|\nabla f^{(k)}(\bar{X}_{t})  -  \nabla f^{(k)}(X_{t}^{(k)})\|_F] + \beta \frac{1}{K}\sum_{k=1}^{K}\mathbb{E}[\|\nabla f^{(k)}(X_{t}^{(k)})-  \nabla f^{(k)}(X_{t}^{(k)}; \xi_{t}^{(k)}) \|_F] \notag \\
        & \leq (1-\beta) \frac{1}{K}\sum_{k=1}^{K}\mathbb{E}[\|\bar{M}_{t-1} - M^{(k)}_{t-1}  \|_F] + 2\beta L\frac{1}{K}\sum_{k=1}^{K}\mathbb{E}[\|\bar{X}_{t}  -  X_{t}^{(k)}\|_F]  + 2\beta\sigma  + \beta\delta \notag \\
        & \leq (1-\beta) \frac{1}{K}\sum_{k=1}^{K}\mathbb{E}[\|\bar{M}_{t-1} - M^{(k)}_{t-1}  \|_F] + 4\beta \eta \tau L  \sqrt{n} + 2\beta\sigma  + \beta\delta \notag \\
        & \leq (4\beta \eta \tau L  \sqrt{n} + 2\beta\sigma  + \beta\delta)\sum_{t'=s_t\tau+1}^{t}(1-\beta)^{t-t'}  \notag \\
        & \leq 4\beta \eta \tau^2 L  \sqrt{n} + 2\beta\tau\sigma  + \beta\tau\delta \ , 
    \end{align}
    where $s_t = \lfloor t/\tau\rfloor$, the fourth step holds due to Assumption~\ref{assumption:regular-noise} and Assumption~\ref{assumption:heterogeneity}, the fifth step holds due to Lemma~\ref{lemma:consensus-error}, and the last step holds due to $\sum_{t'=s_t\tau}^{t-1}(1-\beta)^{t-1-t'}\leq \sum_{t'=s_t\tau}^{t-1}1 \leq \tau $, as $\beta\in (0, 1)$.

\end{proof}

\begin{lemma} \label{lemma:gradient-estimation-error-regular-noise}
	Given Assumptions~\ref{assumption:smoothness},~\ref{assumption:regular-noise},~\ref{assumption:heterogeneity},  by setting $\beta<1$, the following inequality holds:
	\begin{align}
		&  \frac{1}{T}\sum_{t=0}^{T-1} \mathbb{E}[\| \frac{1}{K}\sum_{k=1}^{K} f^{(k)}(X^{(k)}_{t}) - \frac{1}{K}\sum_{k=1}^{K} M^{(k)}_{t}\|_F]   \leq \frac{1}{T} \frac{\sigma}{\beta}+  \frac{\eta \sqrt{n} L}{\beta}  +  \frac{\sqrt{\beta} \sigma}{\sqrt{K}}  \  . 
	\end{align}
\end{lemma}

\begin{proof}
	According to the update of $M^{(k)}_{t}$, we can obtain
	\begin{align}
		& \quad f^{(k)}(X^{(k)}_{t}) -  M^{(k)}_{t} \notag \\
		& = f^{(k)}(X^{(k)}_{t}) - (1-\beta) M^{(k)}_{t-1} -  \beta \nabla f^{(k)}(X_{t}^{(k)}; \xi_{t}^{(k)}) \notag \\
		& =(1-\beta)f^{(k)}(X^{(k)}_{t-1})   - (1-\beta) M^{(k)}_{t-1}  + (1-\beta)  f^{(k)}(X^{(k)}_{t}) - (1-\beta)f^{(k)}(X^{(k)}_{t-1})   \notag \\
		& \quad + \beta  f^{(k)}(X^{(k)}_{t}) -  \beta \nabla f^{(k)}(X_{t}^{(k)}; \xi_{t}^{(k)}) \notag \\
		& = (1-\beta)(f^{(k)}(X^{(k)}_{t-1})- M^{(k)}_{t-1}) +  (1-\beta)( f^{(k)}(X^{(k)}_{t}) -f^{(k)}(X^{(k)}_{t-1})) \notag \\
		& \quad + \beta  (f^{(k)}(X^{(k)}_{t}) -  \nabla f^{(k)}(X_{t}^{(k)}; \xi_{t}^{(k)}) ) \notag \\
		& = (1-\beta)^{t}(f^{(k)}(X^{(k)}_{0})- M^{(k)}_{0}) + \sum_{i=1}^{t}(1-\beta)^{t-i+1}( f^{(k)}(X^{(k)}_{i}) -f^{(k)}(X^{(k)}_{i-1})) \notag \\
		& \quad +  \sum_{i=1}^{t}\beta (1-\beta)^{t-i+1} (f^{(k)}(X^{(k)}_{i}) -  \nabla f^{(k)}(X_{i}^{(k)}; \xi_{i}^{(k)}) ) \ . 
	\end{align}
	
	Then, we can obtain
	\begin{align}
		& \quad \mathbb{E}[\| \frac{1}{K}\sum_{k=1}^{K} f^{(k)}(X^{(k)}_{t}) - \frac{1}{K}\sum_{k=1}^{K} M^{(k)}_{t}\|_F]  \notag \\
		& \leq (1-\beta)^{t} \mathbb{E}[\|\frac{1}{K}\sum_{k=1}^{K} f^{(k)}(X^{(k)}_{0})- \frac{1}{K}\sum_{k=1}^{K} M^{(k)}_{0} \|_F ] \notag \\
		& \quad + \underbrace{\mathbb{E}[\| \frac{1}{K}\sum_{k=1}^{K} \sum_{i=1}^{t}(1-\beta)^{t-i+1} ( f^{(k)}(X^{(k)}_{i}) -f^{(k)}(X^{(k)}_{i-1})) \|_F]}_{T_1}  \notag \\
		& \quad  + \underbrace{\mathbb{E}[\| \frac{1}{K}\sum_{k=1}^{K} \sum_{i=1}^{t}\beta (1-\beta)^{t-i+1} (f^{(k)}(X^{(k)}_{i}) -  \nabla f^{(k)}(X_{i}^{(k)}; \xi_{i}^{(k)}) ) \|_F]}_{T_2} \ . 
	\end{align}
	
	Regarding $T_1$, we have
	\begin{align}
		  T_1 &  = \mathbb{E}[\| \frac{1}{K}\sum_{k=1}^{K} \sum_{i=1}^{t}(1-\beta)^{t-i+1} ( f^{(k)}(X^{(k)}_{i}) -f^{(k)}(X^{(k)}_{i-1})) \|_F ] \notag \\
		& \leq \frac{1}{K}\sum_{k=1}^{K}    \sum_{i=1}^{t}(1-\beta)^{t-i+1} \| f^{(k)}(X^{(k)}_{i}) -f^{(k)}(X^{(k)}_{i-1}) \|_F  \notag \\
		& \leq L \frac{1}{K}\sum_{k=1}^{K}    \sum_{i=1}^{t}(1-\beta)^{t-i+1} \|X^{(k)}_{i} -X^{(k)}_{i-1} \|_F  \notag \\
		& \leq \eta \sqrt{n} L    \sum_{i=1}^{t}(1-\beta)^{t-i+1}   \notag \\
		& \leq   \frac{\eta \sqrt{n} L}{\beta} \ , 
	\end{align}
	where the third step holds due to Assumption~\ref{assumption:smoothness}, and the fourth step holds due to $ \|X^{(k)}_{i} -X^{(k)}_{i-1} \|_F= \eta \| U^{(k)}_{t}(V^{(k)}_{t})^T\|_F \leq \eta \sqrt{n}$. 
	
	Regarding $T_2$, we have
	\begin{align}
		T^2_2 & =   \mathbb{E}[\| \frac{1}{K}\sum_{k=1}^{K} \sum_{i=1}^{t}\beta (1-\beta)^{t-i+1} (f^{(k)}(X^{(k)}_{i}) -  \nabla f^{(k)}(X_{i}^{(k)}; \xi_{i}^{(k)}) ) \|^2_F] \notag \\
		& =\sum_{i=1}^{t}\beta^2 (1-\beta)^{2(t-i+1)} \mathbb{E}[\| \frac{1}{K}\sum_{k=1}^{K}  (f^{(k)}(X^{(k)}_{i}) -  \nabla f^{(k)}(X_{i}^{(k)}; \xi_{i}^{(k)}) ) \|^2_F] \notag \\
		& \leq  \frac{\beta^2 \sigma^2}{K}  \sum_{i=1}^{t}(1-\beta)^{2(t-i+1)} \notag \\
		& \leq \frac{\beta^2 \sigma^2}{K (1-(1-\beta)^2)} \notag \\
		& =  \frac{\beta^2 \sigma^2}{K (2\beta-\beta^2)} \notag \\
		& = \frac{\beta \sigma^2}{K (2-\beta)} \notag \\
		& \leq \frac{\beta \sigma^2}{K} \ , 
	\end{align}
	where the  second step holds due to $\mathbb{E}[ \nabla f^{(k)}(X_{i}^{(k)}; \xi_{i}^{(k)})]=f^{(k)}(X^{(k)}_{i})$,  the third step holds due to Assumption~\ref{assumption:regular-noise}, and the last step holds due to $\beta\in (0, 1)$. 
	
	As a result, we can obtain
	\begin{align}
		&\quad \mathbb{E}[ \| \frac{1}{K}\sum_{k=1}^{K} f^{(k)}(X^{(k)}_{t}) - \frac{1}{K}\sum_{k=1}^{K} M^{(k)}_{t}\|_F]  \notag \\
		& \leq (1-\beta)^{t}\mathbb{E}[ \|\frac{1}{K}\sum_{k=1}^{K} f^{(k)}(X^{(k)}_{0})- \frac{1}{K}\sum_{k=1}^{K} M^{(k)}_{0} \|_F]  +  \frac{\eta \sqrt{n} L}{\beta}  +  \frac{\sqrt{\beta} \sigma}{\sqrt{K}}  \ . 
	\end{align}
	
	By summing over $t$ from $0$ to $T-1$, we have
	\begin{align}
		& \quad \frac{1}{T}\sum_{t=0}^{T-1} \mathbb{E}[\| \frac{1}{K}\sum_{k=1}^{K} f^{(k)}(X^{(k)}_{t}) - \frac{1}{K}\sum_{k=1}^{K} M^{(k)}_{t}\|_F]  \notag \\
		& \leq  \frac{1}{T}\sum_{t=0}^{T-1}  (1-\beta)^{t} \mathbb{E}[\|\frac{1}{K}\sum_{k=1}^{K} f^{(k)}(X^{(k)}_{0})- \frac{1}{K}\sum_{k=1}^{K} M^{(k)}_{0} \|_F]   +  \frac{\eta \sqrt{n} L}{\beta}  +  \frac{\sqrt{\beta} \sigma}{\sqrt{K}}  \notag \\
		& \leq \frac{1}{T} \frac{\sigma}{\beta}+  \frac{\eta \sqrt{n} L}{\beta}  +  \frac{\sqrt{\beta} \sigma}{\sqrt{K}}  \ , 
	\end{align}
	where the last step holds due to $M^{(k)}_{0}=\nabla f^{(k)}(X_{0}^{(k)}; \xi_{0}^{(k)})$ and Assumption~\ref{assumption:regular-noise}. 
\end{proof}

\paragraph{Proof of Theorem~\ref{theorem:regular-noise}.}

\begin{proof}
	Based on Lemma~\ref{lemma:loss-function-decrease}, by summing over $t$ from $0$ to $T-1$, we have
	\begin{align}
	    \frac{1}{T} \sum_{t=0}^{T-1}\mathbb{E}[\| \nabla f(\bar{X}_{t})  \|_F ]&  \leq \frac{	f(\bar{X}_{0})  -    f(\bar{X}_{T})}{\eta T}  + \frac{\eta n L}{2} \notag \\
        & \quad + 2 \sqrt{n} L\frac{1}{T} \sum_{t=0}^{T-1}\frac{1}{K}\sum_{k=1}^{K} \mathbb{E}[\| \bar{X}_{t} - X^{(k)}_{t}\|_F] \notag  \\
        & \quad + 2 \sqrt{n}\frac{1}{T} \sum_{t=0}^{T-1}\frac{1}{K}\sum_{k=1}^{K}\mathbb{E}[\|\bar{M}_{t} -M^{(k)}_{t} \|_F] \notag \\
		& \quad  + 2 \sqrt{n}\frac{1}{T} \sum_{t=0}^{T-1}\mathbb{E}[\| \frac{1}{K}\sum_{k=1}^{K} f^{(k)}(X^{(k)}_{t}) - \frac{1}{K}\sum_{k=1}^{K} M^{(k)}_{t}\|_F] \ . 
	\end{align}

	According to Lemma~\ref{lemma:consensus-error}, Lemma~\ref{lemma:gradient-estimation-error-regular-noise}, and Lemma~\ref{lemma:consensus-error-momentum}, we have
	\begin{align}
		\frac{1}{T} \sum_{t=0}^{T-1}\mathbb{E}[\| \nabla f(\bar{X}_{t})  \|_F ]&  \leq \frac{	f({X}_{0})  -    f({X}_{*})}{\eta T}  + \frac{\eta n L}{2}  + 4 \eta \tau n L   \notag  \\
        & \quad + 8\beta \eta \tau^2 n L   + 4\beta\tau\sqrt{n}\sigma  + 2\beta\tau\sqrt{n}\delta \notag \\
		& \quad  +    \frac{2 \sqrt{n} \sigma}{\beta T}+  \frac{2\eta n L}{\beta}  +  \frac{2 \sqrt{\beta} \sqrt{n} \sigma}{\sqrt{K}}  \ . 
	\end{align}
    By setting $\eta=\frac{K^{1/4}}{T^{3/4}}$,  $\beta=\frac{K^{1/2}}{T^{1/2}}$, and $\tau= \frac{T^{1/4}}{K^{3/4}}$, we can obtain
	\begin{align}
		&\frac{1}{T} \sum_{t=0}^{T-1}\mathbb{E}[\| \nabla f(\bar{X}_{t})  \|_F ]  \leq \frac{	f({X}_{0})  -    f({X}_{*})}{(KT)^{1/4}}  + \frac{K^{1/4} n L}{2T^{3/4}}  +  \frac{4n L }{(KT)^{1/2}}   \notag  \\
        & \quad  + \frac{8nL}{(KT)^{3/4}} + \frac{4\sqrt{n}\sigma}{(KT)^{1/4}} + \frac{2\sqrt{n}\delta}{(KT)^{1/4}}\notag \\
		& \quad  +    \frac{2 \sqrt{n} \sigma}{(KT)^{1/2}}+  \frac{2 n L}{(KT)^{1/4}}  +  \frac{2  \sqrt{n} \sigma}{(KT)^{1/4}}  \notag \\
		& \leq O\left( \frac{	f({X}_{0})  -    f({X}_{*}) + nL + \sqrt{n}\sigma + \sqrt{n}\delta}{(KT)^{1/4}} +  \frac{ nL+\sqrt{n} \sigma}{(KT)^{1/2}}  + \frac{nL}{(KT)^{3/4}} +\frac{K^{1/4} n L}{T^{3/4}}  \right) \ . 
	\end{align}

\end{proof}

\section{Convergence Analysis Under Heavy-Tailed Noise}\label{sec:conv_ht_noise}

To prove Theorem~\ref{theorem:heavy-tailed-noise}, we first introduce an important lemma, originally proved for vectors in \cite{liu2024nonconvex} (see Lemma 4.3), which can be trivially extended to matrices in the following.
\begin{lemma}\label{lemma:zijian-liu-lemma}
    Given random matrices $V_t$ and natural filtration $\mathcal{F}_{t-1}$ for $t\in \mathbb{N}$,  assume that $\mathbb{E}[V_t|\mathcal{F}_{t-1}]=0$. Then, the following inequality holds:
    \begin{align}
        & \mathbb{E}[\|\sum_{t=1}^{T}V_t\|_F]\leq 2\sqrt{2} \mathbb{E}[(\sum_{t=1}^{T}\|V_t\|_F^p)^{\frac{1}{p}}] \ , 
    \end{align}
    where $T\in \mathbb{N}$ and $p\in[1, 2]$. 
\end{lemma}
This is a matrix version of Lemma 4.3 in \cite{liu2024nonconvex}. It can be trivially proved by following the proof in \cite{liu2024nonconvex}. 

\begin{lemma}\label{lemma:loss-function-decrease-hvt}
	Given Assumptions~\ref{assumption:smoothness},~\ref{assumption:heavy-tailed-noise},~\ref{assumption:heterogeneity},  the following inequality holds:
	\begin{align}
		f(\bar{X}_{t+1})&   \leq  f(\bar{X}_{t}) -\eta   \| \nabla f(\bar{X}_{t})  \|_F  + 2\eta \sqrt{n}\frac{1}{K}\sum_{k=1}^{K}\|\bar{M}_{t} -M^{(k)}_{t} \|_F + \frac{\eta^2 n L}{2}\notag \\
        & \quad + 2\eta \sqrt{n} L\frac{1}{K}\sum_{k=1}^{K} \| \bar{X}_{t} - X^{(k)}_{t}\|_F + 2\eta \sqrt{n}\| \frac{1}{K}\sum_{k=1}^{K} f^{(k)}(X^{(k)}_{t}) - \frac{1}{K}\sum_{k=1}^{K} M^{(k)}_{t}\|_F \ . 
	\end{align}
\end{lemma}
This lemma is same as Lemma~\ref{lemma:loss-function-decrease}.

\begin{lemma} \label{lemma:consensus-error-hvt}
	Given Assumptions~\ref{assumption:smoothness},~\ref{assumption:heavy-tailed-noise},~\ref{assumption:heterogeneity}, the following inequality holds:
	\begin{align}
		&  \frac{1}{K}\sum_{k=1}^{K} \| \bar{X}_{t} - X^{(k)}_{t}\|_F \leq 2\eta \tau \sqrt{n}  \ . 
	\end{align}
\end{lemma}
This lemma is same as Lemma~\ref{lemma:consensus-error}. 

\begin{lemma}
   Given Assumptions~\ref{assumption:smoothness},~\ref{assumption:heavy-tailed-noise},~\ref{assumption:heterogeneity},  the following inequality holds:
    \begin{align}
    &  \frac{1}{K}\sum_{k=1}^{K} \mathbb{E}[\|\bar{M}_{t} -M^{(k)}_{t} \|_F]  \leq 4\sqrt{2}\beta\tau\sigma + 4\eta\beta\tau^2  L\sqrt{n} +  \beta\tau \delta \ . 
    \end{align}
\end{lemma}

\begin{proof}
 According to Algorithm~\ref{alg:fedmuon}, we have
\begin{align}
    & \quad \bar{M}_{t} -M^{(k)}_{t} \notag \\
    & = (1-\beta)\bar{M}_{t-1} + \beta\frac{1}{K}\sum_{k'=1}^{K}\nabla f^{(k')}(X_{t}^{(k')}; \xi_{t}^{(k')})   -(1-\beta) M^{(k)}_{t-1} - \beta \nabla f^{(k)}(X_{t}^{(k)}; \xi_{t}^{(k)}) \notag \\
    & = (1-\beta)(\bar{M}_{t-1} - M^{(k)}_{t-1}) + \beta \left(\frac{1}{K}\sum_{k'=1}^{K}\nabla f^{(k')}(X_{t}^{(k')}; \xi_{t}^{(k')}) - \nabla f^{(k)}(X_{t}^{(k)}; \xi_{t}^{(k)})\right) \notag \\
    & = \beta\sum_{i=s_t\tau+1}^{t}(1-\beta)^{t-i}\left(\frac{1}{K}\sum_{k'=1}^{K}\nabla f^{(k')}(X_{i}^{(k')}; \xi_{i}^{(k')}) - \nabla f^{(k)}(X_{i}^{(k)}; \xi_{i}^{(k)})\right) \notag \\
    & = \beta\sum_{i=s_t\tau+1}^{t}(1-\beta)^{t-i}\Big(\frac{1}{K}\sum_{k'=1}^{K}\nabla f^{(k')}(X_{i}^{(k')}; \xi_{i}^{(k')})- \frac{1}{K}\sum_{k'=1}^{K}\nabla f^{(k')}(X_{i}^{(k')}) \Big)\notag \\
    & \quad + \beta\sum_{i=s_t\tau+1}^{t}(1-\beta)^{t-i}\Big( \frac{1}{K}\sum_{k'=1}^{K}\nabla f^{(k')}(X_{i}^{(k')}) - \frac{1}{K}\sum_{k'=1}^{K}\nabla f^{(k')}(\bar{X}_{i}) \Big)\notag \\
    & \quad + \beta\sum_{i=s_t\tau+1}^{t}(1-\beta)^{t-i}\Big(\frac{1}{K}\sum_{k'=1}^{K}\nabla f^{(k')}(\bar{X}_{i}) - \nabla f^{(k)}(\bar{X}_{i}) \Big)\notag \\
    & \quad + \beta\sum_{i=s_t\tau+1}^{t}(1-\beta)^{t-i}\Big( \nabla f^{(k)}(\bar{X}_{i})- \nabla f^{(k)}({X}^{(k)}_{i}) \Big)\notag \\
    & \quad + \beta\sum_{i=s_t\tau+1}^{t}(1-\beta)^{t-i}\Big(\nabla f^{(k)}({X}^{(k)}_{i}) - \nabla f^{(k)}(X_{i}^{(k)}; \xi_{i}^{(k)})\Big)  \ ,
\end{align}
where $s_t = \lfloor t/\tau\rfloor$. 

Then, we can obtain
\begin{align}
    & \quad \frac{1}{K}\sum_{k=1}^{K} \mathbb{E}[\|\bar{M}_{t} -M^{(k)}_{t} \|_F] \notag \\
    & \leq \underbrace{\frac{1}{K}\sum_{k=1}^{K} \mathbb{E}[\|\beta\sum_{i=s_t\tau+1}^{t}(1-\beta)^{t-i}\Big(\frac{1}{K}\sum_{k'=1}^{K}\nabla f^{(k')}(X_{i}^{(k')}; \xi_{i}^{(k')})- \frac{1}{K}\sum_{k'=1}^{K}\nabla f^{(k')}(X_{i}^{(k')}) \Big) \|_F]}_{T_0}\notag \\
    & \quad + \underbrace{\frac{1}{K}\sum_{k=1}^{K} \mathbb{E}[\|\beta\sum_{i=s_t\tau+1}^{t}(1-\beta)^{t-i}\Big( \frac{1}{K}\sum_{k'=1}^{K}\nabla f^{(k')}(X_{i}^{(k')}) - \frac{1}{K}\sum_{k'=1}^{K}\nabla f^{(k')}(\bar{X}_{i}) \Big) \|_F]}_{T_1}\notag \\
    & \quad + \underbrace{\frac{1}{K}\sum_{k=1}^{K} \mathbb{E}[\|\beta\sum_{i=s_t\tau+1}^{t}(1-\beta)^{t-i}\Big(\frac{1}{K}\sum_{k'=1}^{K}\nabla f^{(k')}(\bar{X}_{i}) - \nabla f^{(k)}(\bar{X}_{i}) \Big) \|_F]}_{T_2}\notag \\
    & \quad + \underbrace{\frac{1}{K}\sum_{k=1}^{K} \mathbb{E}[\|\beta\sum_{i=s_t\tau+1}^{t}(1-\beta)^{t-i}\Big( \nabla f^{(k)}(\bar{X}_{i})- \nabla f^{(k)}({X}^{(k)}_{i}) \Big) \|_F]}_{T_3} \notag \\
    & \quad + \underbrace{\frac{1}{K}\sum_{k=1}^{K} \mathbb{E}[\|\beta\sum_{i=s_t\tau+1}^{t}(1-\beta)^{t-i}\Big(\nabla f^{(k)}({X}^{(k)}_{i}) - \nabla f^{(k)}(X_{i}^{(k)}; \xi_{i}^{(k)})\Big) \|_F]}_{T_4} \ . 
\end{align}

Regarding $T_0$, for $p\in (1, 2]$, we have
\begin{align}
    T_0 & = \frac{1}{K}\sum_{k=1}^{K} \mathbb{E}[\|\beta\sum_{i=s_t\tau+1}^{t}(1-\beta)^{t-i}\Big(\frac{1}{K}\sum_{k'=1}^{K}\nabla f^{(k')}(X_{i}^{(k')}; \xi_{i}^{(k')})- \frac{1}{K}\sum_{k'=1}^{K}\nabla f^{(k')}(X_{i}^{(k')}) \Big) \|_F] \notag \\
    & =  \frac{1}{K}\mathbb{E}\left[\left\|\sum_{i=s_t\tau+1}^{t}\sum_{k'=1}^{K}\beta(1-\beta)^{t-i}\left(\nabla f^{(k')}(X_{i}^{(k')}; \xi_{i}^{(k')})- \nabla f^{(k')}(X_{i}^{(k')}) \right) \right\|_F\right] \notag \\
    & \leq  2\sqrt{2} \frac{1}{K}\mathbb{E}\left[\left(\sum_{i=s_t\tau+1}^{t}\sum_{k'=1}^{K}\left\|\beta(1-\beta)^{t-i}\left(\nabla f^{(k')}(X_{i}^{(k')}; \xi_{i}^{(k')})- \nabla f^{(k')}(X_{i}^{(k')}) \right) \right\|^{p}_F\right)^{\frac{1}{p}}\right] \notag \\
    & =  2\sqrt{2} \frac{1}{K}\mathbb{E}\left[\left(\sum_{i=s_t\tau+1}^{t}\sum_{k'=1}^{K}\beta^{p}(1-\beta)^{p(t-i)}\left\|\nabla f^{(k')}(X_{i}^{(k')}; \xi_{i}^{(k')})- \nabla f^{(k')}(X_{i}^{(k')}) \right\|^{p}_F\right)^{\frac{1}{p}}\right] \notag \\
    & \leq  2\sqrt{2} \frac{1}{K}\left(\sum_{i=s_t\tau+1}^{t}\sum_{k'=1}^{K}\beta^{p}(1-\beta)^{p(t-i)}\mathbb{E}\left[\left\|\nabla f^{(k')}(X_{i}^{(k')}; \xi_{i}^{(k')})- \nabla f^{(k')}(X_{i}^{(k')}) \right\|^{p}_F\right]\right)^{\frac{1}{p}} \notag \\
    & \leq   2\sqrt{2} \frac{1}{K}\left(\sum_{i=s_t\tau+1}^{t}\sum_{k'=1}^{K}\beta^{p}(1-\beta)^{p(t-i)}\sigma^p\right)^{\frac{1}{p}} \notag \\
    & =  \frac{2\sqrt{2}\beta\sigma}{K^{1-\frac{1}{p}}}\left(\sum_{i=s_t\tau+1}^{t}(1-\beta)^{p(t-i)}\right)^{\frac{1}{p}} \notag \\
    & \leq 2\sqrt{2}\beta\tau\sigma\ ,
\end{align}
where the third step holds due to Lemma~\ref{lemma:zijian-liu-lemma}, the fifth step holds due to the recursive usage of Hölder’s inequality, the sixth step holds due to Assumption~\ref{assumption:heavy-tailed-noise}, the last step holds due to $\left(\sum_{i=s_t\tau+1}^{t}(1-\beta)^{p(t-i)}\right)^{\frac{1}{p}} \overset{0<\beta<1}{\leq} \left(\sum_{i=s_t\tau+1}^{t}1\right)^{\frac{1}{p}} \leq \tau^{\frac{1}{p}} \overset{\tau>1, 1<p\leq 2}{\leq} \tau$ and $K>1$.

Similarly, regarding $T_4$, for $p\in (1, 2]$, we have
\begin{align}
    T_4 & =\frac{1}{K}\sum_{k=1}^{K} \mathbb{E}[\|\beta\sum_{i=s_t\tau+1}^{t}(1-\beta)^{t-i}\Big(\nabla f^{(k)}({X}^{(k)}_{i}) - \nabla f^{(k)}(X_{i}^{(k)}; \xi_{i}^{(k)})\Big) \|_F] \notag \\
    & \leq 2\sqrt{2}\beta\tau\sigma \ .
\end{align}

Regarding $T_1$, we have
\begin{align}
    T_1 & = \frac{1}{K}\sum_{k=1}^{K} \mathbb{E}[\|\beta\sum_{i=s_t\tau+1}^{t}(1-\beta)^{t-i}\Big( \frac{1}{K}\sum_{k'=1}^{K}\nabla f^{(k')}(X_{i}^{(k')}) - \frac{1}{K}\sum_{k'=1}^{K}\nabla f^{(k')}(\bar{X}_{i}) \Big) \|_F] \notag \\
    & = \mathbb{E}[\|\beta\sum_{i=s_t\tau+1}^{t}(1-\beta)^{t-i}\Big( \frac{1}{K}\sum_{k'=1}^{K}\nabla f^{(k')}(X_{i}^{(k')}) - \frac{1}{K}\sum_{k'=1}^{K}\nabla f^{(k')}(\bar{X}_{i}) \Big) \|_F] \notag \\
    & \leq \beta\sum_{i=s_t\tau+1}^{t}(1-\beta)^{t-i} \frac{1}{K}\sum_{k'=1}^{K}\mathbb{E}[\| \nabla f^{(k')}(X_{i}^{(k')}) -\nabla f^{(k')}(\bar{X}_{i})  \|_F] \notag \\
    & \leq \beta L\sum_{i=s_t\tau+1}^{t}(1-\beta)^{t-i} \frac{1}{K}\sum_{k=1}^{K}\mathbb{E}[\| X_{i}^{(k)} -\bar{X}_{i} \|_F] \notag \\
    & \leq 2\eta\beta\tau  L\sqrt{n} \sum_{i=s_t\tau+1}^{t}(1-\beta)^{t-i}  \notag \\
    & \leq 2\eta\beta\tau^2  L\sqrt{n} \ ,
\end{align}
where the last step holds due to $\beta \in (0, 1)$. 

Similarly, regarding $T_3$, we have
\begin{align}
    & T_3 \leq 2\eta\beta\tau^2  L\sqrt{n} \ . 
\end{align}

Regarding $T_2$, we have
\begin{align}
    T_2 & = \frac{1}{K}\sum_{k=1}^{K} \mathbb{E}[\|\beta\sum_{i=s_t\tau+1}^{t}(1-\beta)^{t-i}\Big(\frac{1}{K}\sum_{k'=1}^{K}\nabla f^{(k')}(\bar{X}_{i}) - \nabla f^{(k)}(\bar{X}_{i}) \Big) \|_F] \notag \\
    & \leq \frac{1}{K}\sum_{k=1}^{K}\beta \sum_{i=s_t\tau+1}^{t}(1-\beta)^{t-i}\mathbb{E}[\|\Big(\frac{1}{K}\sum_{k'=1}^{K}\nabla f^{(k')}(\bar{X}_{i}) - \nabla f^{(k)}(\bar{X}_{i}) \Big) \|_F] \notag \\
    & \leq \delta \beta \sum_{i=s_t\tau+1}^{t}(1-\beta)^{t-i} \notag \\
    & \leq \delta \beta\tau \ . 
\end{align}

As a result, we have
    \begin{align}
    &  \frac{1}{K}\sum_{k=1}^{K} \mathbb{E}[\|\bar{M}_{t} -M^{(k)}_{t} \|_F]  \leq 4\sqrt{2}\beta\tau\sigma + 4\eta\beta\tau^2  L\sqrt{n} +  \beta\tau \delta \ . 
    \end{align}
 
\end{proof}

\begin{lemma} \label{lemma:gradient-estimation-error-heavy-tailed-noise}
	Given Assumptions~\ref{assumption:smoothness},~\ref{assumption:heavy-tailed-noise},~\ref{assumption:heterogeneity},  by setting $\beta<1$, the following inequality holds:
	\begin{align}
		&  \frac{1}{T}\sum_{t=0}^{T-1} \mathbb{E}[\| \frac{1}{K}\sum_{k=1}^{K} f^{(k)}(X^{(k)}_{t}) - \frac{1}{K}\sum_{k=1}^{K} M^{(k)}_{t}\|_F]   \leq \frac{1}{T} \frac{2\sqrt{2}\sigma}{\beta}+  \frac{\eta \sqrt{n} L}{\beta}  +  \frac{2\sqrt{2}\beta^{1-\frac{1}{p}} }{K^{1-\frac{1}{p}}}\sigma  \  . 
	\end{align}
\end{lemma}

\begin{proof}
	Same as the proof of Lemma~\ref{lemma:gradient-estimation-error-regular-noise}, we can obtain
	\begin{align}
		& \quad \mathbb{E}[\| \frac{1}{K}\sum_{k=1}^{K} f^{(k)}(X^{(k)}_{t}) - \frac{1}{K}\sum_{k=1}^{K} M^{(k)}_{t}\|_F]  \notag \\
		& \leq (1-\beta)^{t} \underbrace{\mathbb{E}[\|\frac{1}{K}\sum_{k=1}^{K} f^{(k)}(X^{(k)}_{0})- \frac{1}{K}\sum_{k=1}^{K} M^{(k)}_{0} \|_F ]}_{T_0} \notag \\
		& \quad + \underbrace{\mathbb{E}[\| \frac{1}{K}\sum_{k=1}^{K} \sum_{i=1}^{t}(1-\beta)^{t-i+1} ( f^{(k)}(X^{(k)}_{i}) -f^{(k)}(X^{(k)}_{i-1})) \|_F]}_{T_1}  \notag \\
		& \quad  + \underbrace{\mathbb{E}[\| \frac{1}{K}\sum_{k=1}^{K} \sum_{i=1}^{t}\beta (1-\beta)^{t-i+1} (f^{(k)}(X^{(k)}_{i}) -  \nabla f^{(k)}(X_{i}^{(k)}; \xi_{i}^{(k)}) ) \|_F]}_{T_2} \ . 
	\end{align}
	
	$T_1$ has the same bound as Lemma~\ref{lemma:gradient-estimation-error-regular-noise}:
	\begin{align}
		T_1  \leq   \frac{\eta \sqrt{n} L}{\beta} \ .
	\end{align}

	Regarding $T_0$, for $p\in (1, 2]$, we have
\begin{align}
	T_0  & = \mathbb{E}[\|\frac{1}{K}\sum_{k=1}^{K} f^{(k)}(X^{(k)}_{0})- \frac{1}{K}\sum_{k=1}^{K} M^{(k)}_{0} \|_F ] \notag \\
	& = \frac{1}{K}\mathbb{E}[\|\sum_{k=1}^{K} ( f^{(k)}(X^{(k)}_{0})-   \nabla f^{(k)}(X_{0}^{(k)}; \xi_{0}^{(k)})) \|_F ] \notag \\
    & \leq \frac{2\sqrt{2}}{K}\mathbb{E}\left[\left(\sum_{k=1}^{K}\| ( f^{(k)}(X^{(k)}_{0})-   \nabla f^{(k)}(X_{0}^{(k)}; \xi_{0}^{(k)})) \|^p_F\right)^{\frac{1}{p}} \right] \notag \\
    & \leq \frac{2\sqrt{2}}{K}\left(\sum_{k=1}^{K}\mathbb{E}\left[\| ( f^{(k)}(X^{(k)}_{0})-   \nabla f^{(k)}(X_{0}^{(k)}; \xi_{0}^{(k)})) \|^p_F\right]\right)^{\frac{1}{p}}  \notag \\
    & \leq \frac{2\sqrt{2}}{K^{1-\frac{1}{p}}}\sigma  \notag \\
	& \leq 2\sqrt{2}\sigma \ ,
\end{align}
where the third step holds due to Lemma~\ref{lemma:zijian-liu-lemma}, the fourth step holds due to Hölder’s inequality,  the fifth step holds due to Assumption~\ref{assumption:heavy-tailed-noise}, and the last step holds due to $K>1$.

	Regarding $T_2$, we have
	\begin{align}
		T_2 & =   \mathbb{E}[\| \frac{1}{K}\sum_{k=1}^{K} \sum_{i=1}^{t}\beta (1-\beta)^{t-i+1} (f^{(k)}(X^{(k)}_{i}) -  \nabla f^{(k)}(X_{i}^{(k)}; \xi_{i}^{(k)}) ) \|_F] \notag \\
		& = \frac{1}{K}  \mathbb{E}[\| \sum_{k=1}^{K} \sum_{i=1}^{t}\beta (1-\beta)^{t-i+1} (f^{(k)}(X^{(k)}_{i}) -  \nabla f^{(k)}(X_{i}^{(k)}; \xi_{i}^{(k)}) ) \|_F] \notag \\
		& \leq 2\sqrt{2}\frac{1}{K}  \mathbb{E}[(\sum_{k=1}^{K} \sum_{i=1}^{t}\| \beta (1-\beta)^{t-i+1} (f^{(k)}(X^{(k)}_{i}) -  \nabla f^{(k)}(X_{i}^{(k)}; \xi_{i}^{(k)}) ) \|^p_F)^{\frac{1}{p}}] \notag \\
		& \leq 2\sqrt{2}\frac{1}{K}  \mathbb{E}[(\sum_{k=1}^{K} \sum_{i=1}^{t}  \beta^p (1-\beta)^{p(t-i+1)}\| f^{(k)}(X^{(k)}_{i}) -  \nabla f^{(k)}(X_{i}^{(k)}; \xi_{i}^{(k)})  \|^p_F)^{\frac{1}{p}}] \notag \\
		& \leq 2\sqrt{2}\frac{1}{K} (\sum_{k=1}^{K} \sum_{i=1}^{t}  \beta^p (1-\beta)^{p(t-i+1)} \mathbb{E}[\| f^{(k)}(X^{(k)}_{i}) -  \nabla f^{(k)}(X_{i}^{(k)}; \xi_{i}^{(k)})  \|^p_F])^{\frac{1}{p}} \notag \\
		& \leq 2\sqrt{2}\frac{1}{K} (\sum_{k=1}^{K} \sum_{i=1}^{t}  \beta^p (1-\beta)^{p(t-i+1)}\sigma^{p})^{\frac{1}{p}} \notag \\
		& = \frac{2\sqrt{2}\beta \sigma}{K^{1-\frac{1}{p}}} ( \sum_{i=1}^{t}   (1-\beta)^{p(t-i+1)})^{\frac{1}{p}} \notag \\
		& \leq  \frac{2\sqrt{2}\beta \sigma}{K^{1-\frac{1}{p}}} ( \frac{1}{1-(1-\beta)^{p}})^{\frac{1}{p}} \notag \\
		& \leq  \frac{2\sqrt{2}\beta \sigma}{K^{1-\frac{1}{p}}} ( \frac{1}{1-(1-\beta)})^{\frac{1}{p}} \notag \\
		& \leq  \frac{2\sqrt{2}\beta^{1-\frac{1}{\beta}} }{K^{1-\frac{1}{p}}}\sigma \ , 
	\end{align}
	where  the third step holds due to Lemma~\ref{lemma:zijian-liu-lemma}, the fifth step holds due to Holder’s inequality, the sixth step holds due to Assumption~\ref{assumption:heavy-tailed-noise}. 
	
	As a result, we can obtain
	\begin{align}
		&\quad \mathbb{E}[ \| \frac{1}{K}\sum_{k=1}^{K} f^{(k)}(X^{(k)}_{t}) - \frac{1}{K}\sum_{k=1}^{K} M^{(k)}_{t}\|_F]  \notag \\
		& \leq (1-\beta)^{t}2\sqrt{2}\sigma +  \frac{\eta \sqrt{n} L}{\beta}  +  \frac{2\sqrt{2}\beta^{1-\frac{1}{p}} }{K^{1-\frac{1}{p}}}\sigma  \ . 
	\end{align}
	
	By summing over $t$ from $0$ to $T-1$, we have
	\begin{align}
		& \quad \frac{1}{T}\sum_{t=0}^{T-1} \mathbb{E}[\| \frac{1}{K}\sum_{k=1}^{K} f^{(k)}(X^{(k)}_{t}) - \frac{1}{K}\sum_{k=1}^{K} M^{(k)}_{t}\|_F]  \notag \\
		& \leq  \frac{1}{T}\sum_{t=0}^{T-1}  (1-\beta)^{t} 2\sqrt{2}\sigma  +  \frac{\eta \sqrt{n} L}{\beta}  +  \frac{\sqrt{\beta} \sigma}{\sqrt{K}}  \notag \\
		& \leq \frac{1}{T} \frac{2\sqrt{2}\sigma}{\beta}+  \frac{\eta \sqrt{n} L}{\beta}  +  \frac{2\sqrt{2}\beta^{1-\frac{1}{p}} }{K^{1-\frac{1}{p}}}\sigma  \  . 
	\end{align}
\end{proof}

\paragraph{Proof of Theorem~\ref{theorem:heavy-tailed-noise}.}
\begin{proof}
	Based on Lemma~\ref{lemma:loss-function-decrease-hvt}, by summing over $t$ from $0$ to $T-1$, we have
    \begin{align}
	    \frac{1}{T} \sum_{t=0}^{T-1}\mathbb{E}[\| \nabla f(\bar{X}_{t})  \|_F ]&  \leq \frac{	f(\bar{X}_{0})  -    f(\bar{X}_{T})}{\eta T}  + \frac{\eta n L}{2} \notag \\
        & \quad + 2 \sqrt{n} L\frac{1}{T} \sum_{t=0}^{T-1}\frac{1}{K}\sum_{k=1}^{K} \mathbb{E}[\| \bar{X}_{t} - X^{(k)}_{t}\|_F] \notag  \\
        & \quad + 2 \sqrt{n}\frac{1}{T} \sum_{t=0}^{T-1}\frac{1}{K}\sum_{k=1}^{K}\mathbb{E}[\|\bar{M}_{t} -M^{(k)}_{t} \|_F] \notag \\
		& \quad  + 2 \sqrt{n}\frac{1}{T} \sum_{t=0}^{T-1}\mathbb{E}[\| \frac{1}{K}\sum_{k=1}^{K} f^{(k)}(X^{(k)}_{t}) - \frac{1}{K}\sum_{k=1}^{K} M^{(k)}_{t}\|_F] \ . 
	\end{align}
    


    
	According to Lemma~\ref{lemma:consensus-error-hvt} and Lemma~\ref{lemma:gradient-estimation-error-heavy-tailed-noise}, we have
	\begin{align}
		\frac{1}{T} \sum_{t=0}^{T-1}\mathbb{E}[\| \nabla f(\bar{X}_{t})  \|_F ]&  \leq \frac{	f({X}_{0})  -    f({X}_{*})}{\eta T}  + \frac{\eta n L}{2}  + 4 \eta \tau n L   \notag  \\
        & \quad +8 \eta\beta\tau^2  n L + 8 \sqrt{2}\beta\tau \sqrt{n}\sigma  +  2 \beta\tau \sqrt{n}\delta \notag \\
		& \quad  +   \frac{4 \sqrt{2n}\sigma}{\beta T}+  \frac{2\eta n L}{\beta}  +  \frac{4\sqrt{2n}\beta^{1-\frac{1}{p}} }{K^{1-\frac{1}{p}}}\sigma   \ . 
	\end{align}
	
	
	By setting $\eta=\frac{K^{1/4}}{T^{3/4}}$,  $\beta=\frac{K^{1/2}}{T^{1/2}}$, and $\tau=\frac{T^{1/4}}{K^{3/4}}$, we can obtain
	\begin{align}
		& \frac{1}{T} \sum_{t=0}^{T-1}\mathbb{E}[\| \nabla f(\bar{X}_{t})  \|_F ] \leq \frac{	f({X}_{0})  -    f({X}_{*})}{(KT)^{1/4}}  + \frac{K^{1/4} n L}{2T^{3/4}}  +  \frac{4n L }{(KT)^{1/2}}   \notag  \\
        & \quad  + \frac{8nL}{(KT)^{3/4}} + \frac{8\sqrt{2n}\sigma}{(KT)^{1/4}} + \frac{2\sqrt{n}\delta}{(KT)^{1/4}}\notag \\
		& \quad  +    \frac{4 \sqrt{2n} \sigma}{(KT)^{1/2}}+  \frac{2 n L}{(KT)^{1/4}}  +  \frac{4  \sqrt{2n} \sigma}{(KT)^{\frac{p-1}{2p}}}  \notag \\
		& \leq O\left( \frac{	f({X}_{0})  -    f({X}_{*}) + nL + \sqrt{n}\sigma + \sqrt{n}\delta}{(KT)^{1/4}} +  \frac{ nL +\sqrt{n} \sigma}{(KT)^{1/2}} + \frac{nL}{(KT)^{3/4}} +\frac{K^{1/4} n L}{T^{3/4}} + \frac{  \sqrt{n} \sigma}{(KT)^{\frac{p-1}{2p}}}  \right) \ . 
	\end{align}
\end{proof}


\end{document}